\documentclass{article}
\usepackage{log_2022}
\usepackage[english]{babel}
\usepackage[T1]{fontenc}
\usepackage[utf8]{inputenc}

\usepackage{graphicx}
\usepackage[sort,round]{natbib}

\setcitestyle{authoryear,open={(},close={)}}

\usepackage{subcaption}
\usepackage{booktabs}
\usepackage{url}
\usepackage{todonotes}
\usepackage{multirow}
\usepackage{tikz}
\usetikzlibrary{calc}
\usetikzlibrary{decorations.pathmorphing}
\usetikzlibrary{patterns,arrows,automata,positioning,decorations,shapes,calc}

\tikzstyle{normalvertex}=[circle,fill=white,draw=black]
\tikzstyle{emptyvertex}=[draw,circle,minimum size=7pt,inner sep=0pt]
\tikzstyle{tinyvertex}=[draw,circle,minimum size=3pt,inner sep=0pt]
\tikzstyle{thickedge}=[draw,gray!60,line width=1.6pt,-]

\usepackage{pgfplots}
\pgfplotsset{compat=1.17}
\usepgfplotslibrary{colorbrewer}

\pgfdeclarelayer{background}
\pgfsetlayers{background,main}

\usetikzlibrary{shapes}
\usetikzlibrary{arrows.meta,backgrounds,automata, chains}
\usetikzlibrary{positioning,shapes,matrix,calc}
\tikzstyle{vertex}=[circle, draw, fill=gray!80!white,thick,scale=1.2]
\tikzstyle{edge}=[draw=black, thick,-]

\usepackage{bm}

\usepackage{amsmath}
\usepackage{amssymb}
\usepackage{amsthm}
\usepackage{amsfonts}
\usepackage{thmtools}		
\usepackage{mleftright}
\usepackage{stmaryrd}
\usepackage{nicefrac}
\usepackage{algorithm}
\usepackage{algorithmicx}
\usepackage[noend]{algpseudocode}

\theoremstyle{definition}
\newtheorem{theorem}{Theorem}
\newtheorem{proposition}[theorem]{Proposition}

\newtheorem{lemma}[theorem]{Lemma}
\newtheorem{corollary}[theorem]{Corollary}

\newtheorem{remark}[theorem]{Remark}

\usepackage{thm-restate}
\usepackage[mathic=true]{mathtools}
\usepackage{fixmath}
\usepackage{siunitx}

\usepackage{pifont}

\usepackage{enumitem}
\setlist[enumerate]{itemsep=0.2ex, topsep=0.5\topsep}
\setlist[description]{itemsep=0.2ex, topsep=0.5\topsep}
\setlist[itemize]{itemsep=0.2ex, topsep=0.5\topsep}

\usepackage{setspace}
\usepackage{ellipsis}
\usepackage{xspace}

\usepackage{tcolorbox}

\usepackage[scaled=0.86]{helvet}

\makeatletter
\def\thmt@refnamewithcomma #1#2#3,#4,#5\@nil{%
	\@xa\def\csname\thmt@envname #1utorefname\endcsname{#3}%
	\ifcsname #2refname\endcsname
	\csname #2refname\expandafter\endcsname\expandafter{\thmt@envname}{#3}{#4}%
	\fi
}
\makeatother

\definecolor{purple}{RGB}{147,7,204}
\definecolor{blue}{RGB}{10,153,201}
\definecolor{orange}{RGB}{254,128,41}
\definecolor{gray}{RGB}{239,240,241}
\definecolor{pink}{RGB}{254,15,127}
\definecolor{green}{RGB}{140,211,89}

\definecolor{color1}{RGB}{254,15,127}
\definecolor{color2}{RGB}{10,153,201}
\definecolor{color3}{RGB}{194,145,62}
\definecolor{color4}{RGB}{254,128,41}
\definecolor{color5}{RGB}{254,191,185}
\definecolor{color6}{RGB}{140,211,89}
\definecolor{color7}{RGB}{245,221,66}

\definecolor{purple}{RGB}{147,7,204}
\definecolor{green}{RGB}{5,100,18}
\definecolor{blue}{RGB}{10,153,201}
\definecolor{orange}{RGB}{254,128,41}
\definecolor{gray}{RGB}{239,240,241}

\newcommand{\new}[1]{\emph{#1}}

\newcommand{\mF}{\bm{F}}

\newcommand{\mW}{\bm{W}}

\newcommand{\mA}{\bm{A}}
\newcommand{\mB}{\bm{B}}
\newcommand{\mN}{\bm{N}}
\newcommand{\mM}{\bm{M}}
\newcommand{\mX}{\bm{X}}
\newcommand{\mY}{\bm{Y}}
\newcommand{\mD}{\bm{D}}
\newcommand{\mE}{\bm{E}}
\newcommand{\mP}{\bm{P}}
\newcommand{\mJ}{\bm{J}}
\newcommand{\mZ}{\bm{Z}}
\newcommand{\mI}{\bm{I}}
\newcommand{\mO}{\bm{O}}
\newcommand{\mLambda}{\bm{\Lambda}}

\newcommand{\bbR}{\ensuremath{\mathbb{R}}}

\newcommand{\bbN}{\ensuremath{\mathbb{N}}}

\newcommand{\RR}{\mathbb{R}}

\newcommand{\NN}{\mathbb{N}}

\newcommand{\kwl}{$k$\text{-}\textsf{WL}}
\newcommand{\wlone}{$1$\text{-}\textsf{WL}}
\newcommand{\rwlone}{$1$\text{-}\textsf{RWL}}

\newcommand{\localkwl}{$\delta$-$k$-\textsf{LWL}}
\newcommand{\localmkwl}{$k$-\textsf{RLWL}}

\newcommand{\rgcn}{\textsf{R\text{-}GCN}}
\newcommand{\comp}{\textsf{CompGCN}\xspace}

\newcommand{\krns}[1]{$#1$\text{-}\textsf{RN}s}
\newcommand{\krn}[1]{$#1$\text{-}\textsf{RN}}
\newcommand{\rn}{$k$-\textsf{RN}}

\newcommand{\REL}{\mathsf{RELABEL}}

\newcommand{\hb}{\bm{h}}
\newcommand{\zb}{\bm{z}}

\newcommand{\gb}{\bm{g}}
\newcommand{\bbb}{\bm{b}}

\newcommand{\tup}[1]{{(#1)}}

\newcommand{\UPD}{\mathsf{UPD}}
\newcommand{\AGG}{\mathsf{AGG}}
\newcommand{\RO}{\mathsf{READOUT}}
\newcommand{\RL}{\mathsf{RELABEL}}

\renewcommand{\vec}[1]{\mathbf{#1}}
\newcommand{\oms}{\{\!\!\{}
\newcommand{\cms}{\}\!\!\}}

\newcommand{\sign}{\mathsf{sign}}

\title[Weisfeiler and Leman Go Relational]{Weisfeiler and Leman Go Relational}

\author[Barcelo et al.]{%
    Pablo Barcelo\thanks{Alphabetical author order.}\\
	\institute{Institute for Mathematical and Computational Engineering, PUC Chile \& IMFD Chile \& CENIA Chile}\\
	\email{pbarcelo@uc.cl}\\
	\And
	Mikhail Galkin\\
	\institute{Mila Quebec AI Institute \& McGill University}\\
	\email{mikhail.galkin@mila.quebec}\\
	\And 
	Christopher Morris\\
	\institute{RWTH Aachen University}\\
	\email{morris@cs.rwth-aachen.de}\\
	\And 
	Miguel Romero\\
	\institute{Universidad Adolfo Ibáñez \& CENIA Chile}\\
	\email{miguel.romero.o@uai.cl}\\
}

\begin{document}

\maketitle

\begin{abstract}
	Knowledge graphs, modeling multi-relational data, improve numerous applications such as question answering or graph logical reasoning. Many graph neural networks for such data emerged recently, often outperforming shallow architectures. However, the design of such multi-relational graph neural networks is ad-hoc, driven mainly by intuition and empirical insights. Up to now, their expressivity, their relation to each other, and their (practical) learning performance is poorly understood. Here, we initiate the study of deriving a more principled understanding of multi-relational graph neural networks. Namely, we investigate the limitations in the expressive power of the well-known Relational GCN and Compositional GCN architectures and shed some light on their practical learning performance. By aligning both architectures with a suitable version of the Weisfeiler-Leman test, we establish under which conditions both models have the same expressive power in distinguishing non-isomorphic (multi-relational) graphs or vertices with different structural roles. Further, by leveraging recent progress in designing expressive graph neural networks, we introduce the \krn{k} architecture that provably overcomes the expressiveness limitations of the above two architectures. Empirically, we confirm our theoretical findings in a vertex classification setting over small and large multi-relational graphs.
\end{abstract}

\section{Introduction}
Recently, GNNs~\citep{Gil+2017,Sca+2009} emerged as the most prominent graph representation learning architecture. Notable instances of this architecture include, e.g.,~\citet{Duv+2015,Ham+2017}, and~\citet{Vel+2018}, which can be subsumed under the message-passing framework introduced in~\citet{Gil+2017}. In parallel, approaches based on spectral information were introduced in, e.g.,~\citet{Defferrard2016,Bru+2014,Kip+2017}, and~\citet{Mon+2017}---all of which descend from early work in~\citet{bas+1997,Kir+1995,mic+2005,Mer+2005,mic+2009,Sca+2009} and~\citet{Spe+1997}.

By now, we have a deep understanding of the expressive power of GNNs~\citep{Mor+2021b}. To start with, connections between GNNs and Weisfeiler--Leman type algorithms have been shown. Specifically,~\citet{Mor+2019} and~\citet{Xu+2018b} showed that the $\wlone$ limits the expressive power of any possible GNN architecture in terms of distinguishing non-isomorphic graphs. In turn, these results have been generalized to the \kwl, see, e.g.,~\citet{Azi+2020,Gee+2020a,Gee+2020b,Mar+2019,Mor+2019,Morris2020b,Mor+2022b}, and connected to permutation-equivariant function approximation over graphs, see, e.g.,~\citet{Che+2019,geerts2022,Mae+2019}. \citet{barcelo2020} further established an equivalence between the expressiveness of GNNs with readout functions and $\textsf{C}^2$, the $2$-variable fragment of first-order logic with counting quantifiers.

Most previous works focus on graphs that admit labels on vertices but not edges. However, \new{knowledge} or \new{multi-relational graphs}, that admit
labels on both vertices and edges play a crucial role in numerous applications, such as complex question answering in NLP~\citep{fu2020survey} or visual question answering~\citep{huang2022endowing} in the intersection of NLP and vision. To extract the rich information encoded in the graph's multi-relational structure and its annotations, the knowledge graph community has proposed a large set of \new{relational} GNN architectures, e.g.,~\cite{Sch+2019,Vas+2020,Ye+2022},  tailored towards knowledge or multi-relational graphs, targeting tasks such as vertex and link prediction~\citep{Sch+2019,Zhu+2021,Ye+2022}.
Notably, \citet{Sch+2019} proposed the first architecture, namely, \rgcn, being able to handle multi-relational data. Further, \citet{Vas+2020} proposed an alternative GNN architecture, \comp, using less number of parameters and reported improved empirical performance. In the knowledge graph reasoning area, \rgcn{} and \comp, being strong baselines, spun off numerous improved GNNs for vertex classification and transductive link prediction tasks~\citep{galkin2020message,yu2020generalized,compgcn_random_proj}. They also inspired architectures for more complex reasoning tasks such as inductive link prediction~\citep{teru2020inductive,ali2021improving,Zhu+2021,zhang2022knowledge} and query answering~\citep{daza2020message,starqe,gnn_qe}.

Although these approaches show meaningful empirical performance, their limitations in extracting relevant structural information, their learning performance, and their relation to each other are not understood well. For example, there is no understanding of these approaches' inherent limitations in distinguishing between knowledge graphs with different structural features, explicitly considering the unique properties of multi-relational graphs. Hence, a thorough theoretical investigation of multi-relational GNNs' expressive power and learning performance is yet to be established to become meaningful, vital components in today's knowledge graph reasoning pipeline.

\paragraph{Present Work.} Here, we initiate the study on deriving a principled understanding of the capabilities of GNNs for knowledge or multi-relational graphs. More concretely:
\begin{itemize}
	\item We investigate the expressive power of two well-known GNNs for multi-relation data, \new{Relational GCNs} (\rgcn)~\citep{Sch+2019} and \new{Compositional GCNs} (\comp)~\citep{Vas+2020}. We quantify their limitations by relating them to a suitable version of the established Weisfeiler-Leman graph isomorphism test. In particular, we show under which conditions the above two architectures possess the same expressive power in distinguishing non-isomorphic, multi-relational graphs or vertices with different structural features.
	\item To overcome both architectures' expressiveness limitations, we introduce the \krn{k} architecture, which provably overcomes their limitations and show that increasing $k$ always leads to strictly more expressive architectures.
	\item Empirically, we confirm our theoretical findings on established small- and large-scale multi-relational vertex classification benchmarks.
\end{itemize}

See Subsection~\ref{related_work} in the appendix for an expanded discussion of related work.

\section{Preliminaries}

As usual, let $[n] = \{ 1, \dotsc, n \} \subset \NN$ for $n \geq 1$, and let $\{\!\!\{ \dots\}\!\!\}$ denote a multiset.

A  \new{(undirected) graph} $G$ is a pair $(V(G),E(G))$ with a \emph{finite} set of
\new{vertices} $V(G)$ and a set of \new{edges} $E(G) \subseteq \{ \{u,v\}
	\subseteq V \mid u \neq v \}$.
For notational convenience, we usually denote an edge $\{u,v\}$ in $E(G)$ by $(u,v)$ or $(v,u)$.
We assume the usual definition of \new{adjacency matrix} $\mA$ of $G$. A \new{colored} or \new{labeled graph} $G$ is a triple $(V(G),E(G),\ell)$ with a \new{coloring} or \new{label} function $\ell \colon V(G)  \to \bbN$.  Then $\ell(w)$ is a \new{color} or \new{label} of $w$, for $w$ in $V(G)$.  The \new{neighborhood} of $v$ in $V(G)$ is denoted by $N(v) = \{ u \in V(G) \mid (v, u) \in E(G) \}$.

An \new{(undirected) multi-relational graph} $G$ is a tuple $(V(G),R_1(G), \dots, R_r(G))$ with a \emph{finite} set of
\new{vertices} $V(G)$ and \new{relations} $R_i \subseteq \{ \{u,v\} \subseteq V(G) \mid u \neq v \}$ for $i$ in $[r]$. The \new{neighborhood} of $v$ in $V(G)$ with respect to the relation $R_i$ is denoted by $N_i(v) = \{ u \in V(G) \mid (v, u) \in R_i \}$. We define \new{colored} (or \new{labeled})
multi-relational graphs in the expected way.

Two graphs $G$ and $H$ are \new{isomorphic} ($G \simeq H$) if there exists a bijection $\varphi \colon V(G) \to V(H)$ preserving the adjacency relation, i.e., $(u,v)$ in $E(G)$ if and only if $(\varphi(u),\varphi(v))$ in $E(H)$. We then call $\varphi$ an \emph{isomorphism} from $G$ to $H$. If the graphs have vertex labels, the isomorphism is additionally required to match these labels. In the case of multi-relational graphs $G$ and $H$, the bijection $\varphi \colon V(G) \to V(H)$ needs to preserve all relations, i.e., $(u,v)$ is in $R_i(G)$ if and only if $(\varphi(u),\varphi(v))$ is in $R_i(H)$ for each $i$ in $[r]$. For labeled multi-relational graphs, the bijection needs to preserve the labels.

We define the atomic type $\mathsf{atp} \colon V(G)^k \to \bbN$ such that $\mathsf{atp}(\vec{v}) = \mathsf{atp}(\vec{w})$ for $\vec{v}$ and $\vec{w}$ in $V(G)^k$ if and only if the mapping $\varphi\colon V(G) \to V(G)$ where $v_i \mapsto w_i$ induces a \emph{partial isomorphism}, i.e., $v_i = v_j \iff w_i = w_j$ and $(v_i,v_j)$ in $E(G) \iff (\varphi(v_i),\varphi(v_j))$ in $E(G)$.

\paragraph{The Weisfeiler-Leman Algorithm.}\label{vr_ext}
The \new{$1$-dimensional Weisfeiler-Leman algorithm} (\wlone), or \new{color refinement}, is a simple heuristic for the graph isomorphism problem, originally proposed by~\citet{Wei+1968}.\footnote{Strictly speaking, \wlone\ and color refinement are two different algorithms. That is, \wlone\ considers neighbors and non-neighbors to update the coloring, resulting in a slightly higher expressive power when distinguishing vertices in a given graph, see~\citet{Gro+2021} for details. For brevity, we consider both algorithms to be equivalent.}\footnote{
	We use the spelling ``Leman'' here as A.~Leman, co-inventor of the algorithm, preferred it over the transcription ``Lehman''; see
	\url{https://www.iti.zcu.cz/wl2018/pdf/leman.pdf}.}
Intuitively, the algorithm determines if two graphs are non-isomorphic by iteratively coloring or labeling vertices. Given an initial coloring or labeling of the vertices of both
graphs, e.g., their degree or application-specific information, in each iteration, two vertices with the same label get different labels if the number of identically labeled neighbors is not equal. If, after some iteration, the number of vertices annotated with a specific label is different in both graphs, the algorithm terminates and a stable coloring, inducing a vertex partition, is obtained. We can then conclude that the two graphs are not isomorphic. It is easy to see that the algorithm cannot distinguish all non-isomorphic graphs~\citep{Cai+1992}. Nonetheless, it is a powerful heuristic that can successfully test isomorphism for a broad class of graphs~\citep{Arv+2015,Bab+1979,Kie+2015}.

Formally, let $G = (V(G),E(G),\ell)$ be a labeled graph. In each iteration, $t > 0$, the $\wlone$ computes a vertex coloring $C^{(t)} \colon V(G) \to \bbN$,
which depends on the coloring of the neighbors. That is, in iteration $t>0$, we set
\begin{equation*}
	C^{(t)}(v) \coloneqq \REL\Big(\!\big(C^{(t-1)}(v),\oms C^{(t-1)}(u) \mid u \in N(v)  \cms \big)\! \Big),
\end{equation*}
where $\REL$ injectively maps the above pair to a unique natural number, which has not been used in previous iterations.  In iteration $0$, the coloring $C^{(0)}\coloneqq \ell$. To test if two graphs $G$ and $H$ are non-isomorphic, we run the above algorithm in ``parallel'' on both graphs. If the two graphs have a different number of vertices colored $c$ in $\bbN$ at some iteration, the \wlone{} \new{distinguishes} the graphs as non-isomorphic. Moreover, if the number of colors between two iterations, $t$ and $(t+1)$, does not change, i.e., the cardinalities of the images of $C^{(t)}$ and $C^{(t+1)}$ are equal, or, equivalently,
\begin{equation*}
	C^{(t)}(v) = C^{(t)}(w) \iff C^{(t+1)}(v) = C^{(t+1)}(w),
\end{equation*}
for all vertices $v$ and $w$ in $V(G)$, the algorithm terminates. For such $t$, we define the \new{stable coloring}
$C^{\infty}(v) = C^{(t)}(v)$ for $v$ in $V(G)$. The stable coloring is reached after at most $\max \{ |V(G)|,|V(H)| \}$ iterations~\citep{Gro2017}.

Due to the shortcomings of the $\wlone$ in distinguishing non-isomorphic graphs, several researchers, e.g.,~\citep{Bab1979,Cai1992}, devised a more powerful generalization of the former, today known as the $k$-dimensional Weisfeiler-Leman algorithm (\kwl), see Subsection~\ref{APP:vr_ext} for details.

\paragraph{Graph Neural Networks.}\label{sec:gnn} Intuitively, GNNs learn a vectorial representation, i.e., a $d$-dimensional vector, representing each vertex in a graph by aggregating information from neighboring vertices. Formally, let $G = (V(G),E(G),\ell)$ be a labeled graph with initial vertex features $(\hb_{v}^\tup{0})_{v\in V(G)}$ in $\RR^{d}$ that are \emph{consistent} with $\ell$, that is, $\hb_{u}^\tup{0} = \hb_{v}^\tup{0}$ if and only if $\ell(u) = \ell(v)$, e.g., a one-hot encoding of the labelling $\ell$. Alternatively, $(\hb_{v}^\tup{0})_{v\in V(G)}$ can be arbitrary vertex features annotating the vertices of $G$.

A GNN architecture consists of a stack of neural network layers, i.e., a composition of permutation-invariant or -equivariant parameterized functions. Similarly to \wlone, each layer aggregates local neighborhood information, i.e., the neighbors' features, around each vertex and then passes this aggregated information on to the next layer.

GNNs are often realized as follows~\citep{Mor+2019}. In each layer, $t > 0$, we compute vertex features
\begin{equation}\label{eq:basicgnn}
	\hb_{v}^\tup{t} \coloneqq \sigma \Big( 	\hb_{v}^\tup{t-1} \mW^{(t)}_0 + \sum_{{w \in N(v)}} \hb_{w}^\tup{t-1} \mW_1^{(t)} \Big) \in \bbR^{e},
\end{equation}
for $v$ in $V(G)$, where
$\mW_0^{(t)}$ and $\mW_1^{(t)}$ are parameter matrices from $\bbR^{d \times e}$ and $\sigma$ denotes an entry-wise non-linear function, e.g., a sigmoid or a ReLU function.\footnote{For clarity of presentation, we omit biases.} Following \citet{Gil+2017} and \citet{Sca+2009}, in each layer, $t > 0$, we can generalize the above by computing a vertex feature
\begin{equation*}\label{def:gnn}
	\hb_{v}^\tup{t} \coloneqq
	\UPD^\tup{t}\Bigl(\hb_{v}^\tup{t-1},\AGG^\tup{t} \bigl(\oms \hb_{w}^\tup{t-1}
	\mid w\in N(v) \cms \bigr)\Bigr), 
\end{equation*}
where  $\UPD^\tup{t}$ and $\AGG^\tup{t}$ may be differentiable parameterized functions, e.g., neural networks.\footnote{Strictly speaking, \citet{Gil+2017} consider a slightly more general setting in which vertex features are computed by $\hb_{v}^\tup{t+1} \coloneqq
		\UPD^\tup{t+1}\Bigl(\hb_{v}^\tup{t},\AGG^\tup{t+1} \bigl(\oms (\hb_v^\tup{t},\hb_{w}^\tup{t},\ell(v,w))
		\mid w\in N(v) \cms \bigr)\Bigr)$.}
In the case of graph-level tasks, e.g., graph classification, one uses
\begin{equation*}\label{readout}
	\hb_G \coloneqq \RO\bigl( \oms \hb_{v}^{\tup{T}}\mid v\in V(G) \cms \bigr), 
\end{equation*}
to compute a single vectorial representation based on learned vertex features after iteration $T$. Again, $\RO$  may be a differentiable parameterized function. To adapt the parameters of the above three functions, they are optimized end-to-end, usually through a variant of stochastic gradient descent, e.g.,~\citep{Kin+2015}, together with the parameters of a neural network used for classification or regression.

\paragraph{Graph Neural Networks for Multi-relational Graphs.}
In the following, we describe GNN layers for multi-relational graphs, namely \rgcn{}~\citep{Sch+2019} and \comp{}~\citep{Vas+2020}. Initial features are computed in the same way as in the previous subsection.

\paragraph{\textsf{R\text{-}GCN}} Let $G$ be a labeled 
multi-relational graph. In essence, \rgcn\ generalizes Equation~\ref{eq:basicgnn} by using an additional sum iterating over the different relations. That is, we compute a vertex feature
\begin{align}\label{eq:rgcn}
	\hb_{v, \rgcn}^\tup{t} \coloneqq \sigma \Big( 	\hb_{v, \rgcn}^\tup{t-1}  \mW^{(t)}_0 +  \sum_{i \in [r]}  \sum_{{w \in N_i(v)}}\! \hb_{w,\rgcn}^\tup{t-1} \mW_i^{(t)} \Big)  \in \RR^{e},
\end{align}
for $v$ in $V(G)$, where $\mW_0^{(t)}$ and $\mW_i^{(t)}$ for $i$ in $[r]$ are parameter matrices from $\bbR^{d \times e}$, and $\sigma$ denotes a entry-wise non-linear function. We note here that the original \rgcn{} layer defined in~\cite{Sch+2019} uses a mean operation instead of a sum in the most inner sum of Equation~\ref{eq:rgcn}. We investigate the empirical advantages of these two variations in Section~\ref{sec:experiments}.

\paragraph{\textsf{CompGCN}}  Let $G$ be a labeled 
multi-relational graph.  A \comp layer generalizes~ Equation~\ref{eq:basicgnn} by encoding relational information as edge features. That is, we compute a vertex feature
\begin{align}\label{eq:comp}
	\hb_{v, \comp}^\tup{t} \coloneqq \sigma \Big(\hb_{v, \comp}^\tup{t-1} \mW^{(t)}_0 +
	\sum_{i \in [r]} \sum_{w \in N_i(v)}  \phi \big(\hb_{w, \comp}^\tup{t-1}, \zb_i^{(t)} \big)  \mW_1^{(t)} \Big)\in \RR^{e},
\end{align}
for $v$ in $V(G)$, where $\mW_0^{(t)}$ and $\mW_1^{(t)}$ are parameter matrices from $\bbR^{d \times e}$ and $\bbR^{c \times e}$, respectively, and $\zb_i^{(t)}$ in $\bbR^{b}$ is the learned edge feature for the $i$th relation at layer $t$. Further, the function $\phi \colon \bbR^d \times  \bbR^b \to \bbR^c$
is a \new{composition map}, mapping two vectors onto a single vector in a non-parametric way, e.g., summation, point-wise multiplication, or concatenation. We note here that the original \comp layer defined in~\cite{Vas+2020} uses an additional sum to differentiate between in-going and out-going edges and self loops, see Section~\ref{APP:comp} for details.

\section{Relational Weisfeiler--Leman Algorithm}\label{sec:multi-wl}

In the following, to study the limitations in expressivity of the above two GNN layers, \rgcn\ and \comp, we define the \new{multi-relational \wlone} (\rwlone). Let $G=(V(G),R_1(G),\dots, R_r(G),\ell)$ be a labeled, multi-relational graph. Then the \rwlone{} computes a vertex coloring $C^{(t)}_{\textsf{R}} \colon V(G) \to \bbN$ for $t > 0$ by interpreting the different relations as edge types, i.e.,
\begin{align}\label{wlscomp}
	C^{(t)}_{\textsf{R}}(v) \coloneqq \RL \Big(\!\big(C^{(t-1)}_{\textsf{R}}(v), \oms (C^{(t-1)}_{\textsf{R}}(u), i) \mid i \in [r], u \in\!N_i(v) \cms \big)\! \Big),
\end{align}
for $v$ in $V(G)$.
In iteration $0$, the coloring $C^{(0)}_{\textsf{R}} \coloneqq \ell$.
In particular, two vertices $v$ and $w$ of the same color in iteration $(t-1)$ get different colors in iteration $t$ if there is a relation $R_i$ such that the number of neighbors in $N_i(v)$ and $N_i(w)$ colored with a certain color is different. We define the stable coloring $C^{\infty}_{\textsf{R}}$ in the expected way, analogously to the \wlone.

\paragraph{Relationship Between \wlone{} and \rwlone} Since \wlone{} does not consider edge labels it is clear that \rwlone{} is strictly stronger than the \wlone{}. For example, take a pair of isomorphic graphs and label the edges differently in each graphs, making the graph non-isomorphic. Clearly, \rwlone{} will distinguish them while \wlone{} will not.

\paragraph{Relationship Between \rwlone, \rgcn, and \comp\!\!}
\citet{Mor+2019,Xu+2018b} established the exact relationship between the expressive power of \wlone{} and GNNs.
In particular, \wlone{} upper bounds the capacity of any GNN architecture for distinguishing vertices in graphs. In turn,
over every graph $G$ there is a GNN architecture with the same expressive power as $\wlone$ for distinguishing vertices in $G$. In this section, we show that the same relationship can be established between multi-relational $\wlone$, on the one hand, and the \rgcn\ and \comp architectures, on the other.

Let $G=(V(G),R_1(G), \dots, R_r(G),\ell)$ be a labeled, multi-relational graph, and let $$\mathbf{W}_\rgcn^\tup{t} \ = \ \big(\mW^{(t')}_0, \mW^{(t')}_{i} \big)_{t'\leq t, i \in [r]}$$
denote the sequence of \rgcn\ parameters given by Equation~\ref{eq:rgcn} up to iteration $t$.
Analogously, we denote by $$\mathbf{W}_\comp^\tup{t} \ = \ (\mW^{(t')}_0, \mW^{(t')}_{1},\zb_i^{(t')})_{t'\leq t, i \in [r]}$$ the sequence of
\comp parameters given by Equation~\ref{eq:comp} up to iteration $t$. We first show that the multi-relational \wlone{} upper bounds the expressivity of both the \rgcn{} and \comp{} layers in terms of their capacity to distinguish vertices in labeled
multi-relational graphs.
\medskip
\begin{theorem}
	\label{thm:mrgcn-upper}
	{\em Let $G = (V(G), R_1(G), \dots, R_r(G), \ell)$ be a labeled, multi-relational graph. Then for all $t \geq 0$ the following holds:
		\begin{itemize}
			\item For all choices of initial vertex features consistent with $\ell$, sequences $\mathbf{W}_\rgcn^{(t)}$ of \rgcn\ parameters,
			      and vertices $v$ and $w$ in $V(G)$,
			      \begin{equation*}
				      C^{(t)}_{\textsf{R}}(v) =  	C^{(t)}_{\textsf{R}}(w) \ \ \Longrightarrow \ \   \hb_{v, \rgcn}^\tup{t} = \hb_{w, \rgcn}^\tup{t}.
			      \end{equation*}
			\item For all choices of initial vertex features consistent with $\ell$, sequences $\mathbf{W}_\comp^{(t)}$ of \comp
			      parameters,
			      composition functions $\phi$, and vertices $v$ and $w$ in $V(G)$,
			      \begin{equation*}
				      C^{(t)}_{\textsf{R}}(v) =  	C^{(t)}_{\textsf{R}}(w) \ \ \Longrightarrow \ \  \hb_{v, \comp}^\tup{t} = \hb_{w, \comp}^\tup{t}.
			      \end{equation*}
		\end{itemize} }
\end{theorem}
Noticeably, the converse also holds. That is, there is a sequence of parameter matrices $\mathbf{W}_\rgcn^{(t)}$ such that \rgcn\ has the same expressive power in terms of distinguishing vertices in graphs as the coloring $C^{(t)}_{\textsf{R}}$. This equivalence holds provided the initial labels are encoded by linearly independent vertex features, e.g., using one-hot encodings. The result also holds for \comp as long as the composition map $\phi$ can express vector scaling, e.g., $\phi$ is point-wise multiplication or circular correlation, two of the composition functions studied and implemented in the paper that introduced the \comp architecture \citep{Vas+2020}.
\medskip
\begin{theorem}\label{thm:mrgcn-lower}
	{\em Let $G=(V(G), R_1(G), \dots, R_r(G), \ell)$ be a labeled, multi-relational graph. Then for all $t \geq 0$ the following holds:
		\begin{itemize}
			\item
			      There are initial vertex features and a sequence $\mathbf{W}_\rgcn^{(t)}$ of parameters such that for all $v$ and $w$ in $V(G)$,
			      \begin{equation*}
				      C^{(t)}_{\textsf{R}}(v) =  	C^{(t)}_{\textsf{R}}(w) \ \ \Longleftrightarrow \ \   \hb_{v, \rgcn}^\tup{t} = \hb_{w, \rgcn}^\tup{t}.
			      \end{equation*}
			\item
			      There are initial vertex features, a sequence $\mathbf{W}_\comp^{(t)}$ of parameters
			      and a composition function $\phi$ such that for all $v$ and $w$ in $V(G)$,
			      \begin{equation*}
				      C^{(t)}_{\textsf{R}}(v) =  	C^{(t)}_{\textsf{R}}(w) \ \ \Longleftrightarrow \ \   \hb_{v, \comp}^\tup{t} = \hb_{w, \comp}^\tup{t}.
			      \end{equation*}
		\end{itemize} }
\end{theorem}

\paragraph{On the Choice of the Composition Function for \comp{} Architectures}
As Theorem~\ref{thm:mrgcn-lower} shows the expressive power of the  $\rwlone$ is matched by that of the \comp\ architectures if we allow the latter to implement vector scaling in composition functions.
However, not all composition maps that have been considered in relationship with \comp\ architectures admit such a possibility.
Think, for instance, of natural composition maps such as point-wise summation or vector concatenation. Interestingly, we can show that \comp\ architectures equipped with these composition maps are provably weaker in terms of expressive power than the ones studied in the proof of Theorem~\ref{thm:mrgcn-lower}, as they correlate with a weaker variant of $\wlone$ that we define next.

Let $G=(V(G),R_1(G),\dots, R_r(G),\ell)$ be a labeled, multi-relational graph. The \emph{weak multi-relational \wlone}  computes a vertex coloring $C^{(t)}_{\textsf{WR}} \colon V(G) \to \bbN$ for $t > 0$ as follows:
\begin{equation*}\label{eq:weak-wlscomp}
	C^{(t)}_{\textsf{WR}}(v) \coloneqq \RL \Big(\!\big(C^{(t-1)}_{\textsf{WR}}(v), \oms C^{(t-1)}_{\textsf{WR}}(u) \mid  i \in [r], u \in\! N_i(v)\cms, |N_1(v)|, \dots, |N_r(v)| \big)\! \Big),
\end{equation*}
for $v$ in $V(G)$. In iteration $0$, the coloring $C^{(0)}_{\textsf{WR}} \coloneqq \ell$. During aggregation, the weak variant does not take information about the relations into account. The only information relative to the different relations is the number of neighbors associated with each of them. We define the stable coloring $C^{\infty}_{\textsf{WR}}$ analogously to the \wlone. As it turns out, this variant is less powerful than the original one.
\medskip
\begin{proposition}
	\label{prop:weak-wl-sep}
	{\em
		There exist a labeled, multi-relational graph $G=(V(G),R_1(G),R_2(G),\ell)$ and two vertices $v$ and $w$ in $V(G)$, such that $C^{(1)}_{\textsf{R}}(v)\neq C^{(1)}_{\textsf{R}}(w)$ but $C^{\infty}_{\textsf{WR}}(v) = C^{\infty}_{\textsf{WR}}(w)$.
	}
\end{proposition}
As shown next, the expressive power of \comp\ architectures that use point-wise summation
or vector concatenation is captured by this
weaker form of \rwlone.
\medskip
\begin{theorem}\label{thm:weak-comp-wl}
	{\em
		Let $G = (V(G), R_1(G), \dots, R_r(G), \ell)$ be a labeled, multi-relational graph. Then for all $t \geq 0$ the following holds:
		\begin{itemize}
			\item For all choices of initial vertex features consistent with $\ell$,
			      sequences $\mathbf{W}_\comp^{(t)}$ of \comp\ parameters, and vertices
			      $v$ and $w$ in $V(G)$,
			      \begin{equation*}
				      C^{(t)}_{\textsf{WR}}(v) =  	C^{(t)}_{\textsf{WR}}(w) \ \ \Longrightarrow \ \  \hb_{v, \comp}^\tup{t} = \hb_{w, \comp}^\tup{t},
			      \end{equation*}
			      for either point-wise summation or concatenation as the composition map.

			\item There exist initial vertex features and a sequence $\mathbf{W}_\comp^{(t)}$ of \comp\
			      parameters, such that for all vertices $v$ and $w$ in $V(G)$,
			      \begin{equation*}
				      C^{(t)}_{\textsf{WR}}(v) =  	C^{(t)}_{\textsf{WR}}(w) \ \ \Longleftrightarrow \ \   \hb_{v, \comp}^\tup{t} = \hb_{w, \comp}^\tup{t},
			      \end{equation*}
			      for either point-wise summation or concatenation as the composition map.
		\end{itemize}
	}
\end{theorem}

Together with Proposition \ref{prop:weak-wl-sep} and Theorem \ref{thm:mrgcn-lower}, this result states that \comp\
architectures based on vector summation or concatenation  are provably weaker in terms of their capacity to distinguish vertices in graphs than the ones that use vector scaling.

We have shown that \rgcn{} and \comp{} with point-wise multiplication have the same expressive power in terms of distinguishing non-isomorphic multi-relational graphs or distinguishing vertices in a multi-relational graph. As it turns out, these two architectures actually define the \emph{same} functions. A similar result holds between \comp{} with vector summation/subtraction and concatenation. See Appendix \ref{APP:sec:comparison} for details.





\section{Limitations and More Expressive Architectures}\label{sec:more_exp}

Theorem~\ref{thm:mrgcn-upper} shows that both \rgcn{} as well as \comp have severe limitations in distinguishing structurally different multi-relational graphs. Indeed, the following results shows that there exists pairs of non-isomorphic, multi-relational graphs that neither \rgcn{} nor \comp can distinguish.
\medskip
\begin{proposition}\label{upper}
	{\em For all $r \geq 1$, there exists a pair of non-isomorphic graphs $G = (V(G), R_1(G), \dots, R_r(G), \ell)$ and $H= (V(H), R_1(H), \dots, R_r(H), \ell)$ that cannot be distinguished by \rgcn{} or \comp. }
\end{proposition}
We note here that the two graphs $G$ and $H$ from the above theorem can also be used to show that neither \rgcn{} nor \comp
will be able to compute different features for vertices in $G$ and $H$, making them indistinguishable. Hence, to overcome the limitations of the \comp{} and \rgcn, we introduce \emph{local $k$-order relational networks} (\krns{k}), leveraging recent progress in overcoming GNNs' inherent limitations in expressive power~\citep{Mor+2019,Morris2020b,Mor+2021b,Mor+2022b}. To do so, we first extend the local $k$-dimensional Weisfeiler--Leman algorithm~\citep{Morris2020b}, see Subsection~\ref{APP:vr_ext}, to multi-relational graphs.

\paragraph{Multi-relational Local \kwl} Given a multi-relational graph $G = (V(G), R_1(G), \dots, R_r(G), \ell)$, we define the \new{multi-relational atomic type} $\mathsf{atp}_r \colon V(G)^k \to \bbN$ such  that $\mathsf{atp}_r(\vec{v}) = \mathsf{atp}_r(\vec{w})$ for $\vec{v}$ and $\vec{w}$ in $V(G)^k$ if and only if the mapping $\varphi\colon V(G) \to V(G)$ where $v_g \mapsto w_g$ induces a partial isomorphism, preserving the relations, i.e., we have $v_p = v_q \iff w_p = w_q$ and $(v_p,v_q) \in R_i(G) \iff (\varphi(v_p),\varphi(v_q)) \in R_i(G)$ for $i$ in $[r]$. The \new{multi-relational local  \kwl} (\localmkwl) computes the coloring $C_{k,r}^{(t)} \colon V(G)^k \to \bbN$ for $t \geq 0$, where $C_{k,r}^{(0)} \coloneqq  \mathsf{atp}_r(\vec{v})$, and refines a coloring $C_{k,r}^{(t)}$ (obtained after $t$ iterations of the \localmkwl) via the \new{aggregation function}
\begin{equation}\label{eqnmidd_ext_m}
	\begin{split}
		M_{r}^{(t)}(\vec{v}) \coloneqq   \big( &\{\!\! \{ (C_{k,r}^{(t)}(\theta_1(\vec{v},w)), i) \mid w \in N_i(v_1)  \text{ and } i \in [r]  \} \!\!\}, \dots, \\ &\{\!\! \{  (C_{k,r}^{(t)}(\theta_k(\vec{v},w)),i) \mid w \in N_i(v_k) \text{ and } i \in [r] \}  \!\!\} \big),
	\end{split}
\end{equation}
where $\theta_j(\vec{v},w)\coloneqq (v_1, \dots, v_{j-1}, w, v_{j+1}, \dots, v_k)$. That is, $\theta_j(\vec{v},w)$ replaces the $j$-th component of the tuple $\vec{v}$ with the vertex $w$. Like the local \kwl~\citep{Morris2020b}, the algorithm considers only the local $j$-neighbors, i.e., $v_i$ and $w$ must be adjacent, for each relation in each iteration and additionally differentiates between different relations. The coloring functions for the iterations of the multi-relational \localmkwl{} are then defined by
\begin{equation*}\label{wlsimple_ext_m}
	C_{k,r}^{(t+1)}(\vec{v}) \coloneqq (C_{k,r}^{(t)}(\vec{v}), M_{r}^{(t)}(\vec{v})).
\end{equation*}
In the following, we derive a neural architecture, the \krn{k}, that has the same expressive power as the  \localmkwl{} in terms of distinguishing non-isomorphic multi-relational graphs.

\paragraph{The \krn{k} Architecture.} Given a labeled, multi-relational graph $G$, for each $k$-tuple $\vec{v}$ in $V(G)^k$, a \krn{k} architecture computes an initial feature $\hb_{v}^\tup{0}$ \new{consistent} with its multi-relational atomic type, e.g., a one-hot encoding of $\mathsf{atp}_r(\vec{v})$. In each layer, $t > 0$, a \krn{k} computes a $k$-tuple feature
\begin{align}\label{kgnn}
	\begin{split}
		\hb_{\vec{v},k}^\tup{t} \coloneqq
		\UPD^\tup{t}\Bigl( \hb_{\vec{v},k}^\tup{t-1},\AGG^\tup{t} \bigl(&\oms \phi(\hb_{\theta_1(\vec{v},w),k}^\tup{t-1}, \zb_i^{(t)}) \
		\mid w \in N_i(v_1)  \text{ and } i \in [r] \cms, \dots,\\ &\oms \phi(\hb_{\theta_k(\vec{v},w),k}^\tup{t-1}, \zb_i^{(t)})
		\mid w \in N_i(v_k)  \text{ and } i \in [r] \cms \bigr)\Bigr) \in \RR^{e},
	\end{split}
\end{align}
where the functions $\UPD^\tup{t}$ and $\AGG^\tup{t}$ for $t > 0$ may be a differentiable parameterized functions, e.g., neural networks.
Similarly to Equation~\ref{eq:comp}, $\zb_i^{(t)}$ in $\bbR^{c}$ is the learned edge feature for the $i$th relation at layer $t$ and $\phi \colon \bbR^d \times  \bbR^b \to \bbR^c$ is a composition map. In the case of graph-level tasks, e.g., graph classification, one uses
\begin{equation}\label{readout_k}
	\hb_G \coloneqq \RO\bigl( \oms \hb_{\vec{v}}^{\tup{T}}\mid \vec{v} \in V(G)^k \cms \bigr) \in \RR^{e},
\end{equation}
to compute a single vectorial representation based on learned $k$-tuple features after iteration $T$. The following results shows that the \localmkwl{} upperbounds the expressivity of any \krn{k} in terms of distinguishing non-isomorphic graphs.
\medskip
\begin{proposition}
	\label{thm:ho-upper}
	{\em	Let $G = (V(G), R_1(G), \dots, R_r(G), \ell)$ be a labeled, multi-relational graph. Then for all $t \geq 0$, $r > 0$, $k \geq 1$, and all choices of $\UPD^\tup{t}$, $\AGG^\tup{t}$, and all $\vec{v}$ and $\vec{w}$ in $V(G)^k$,
		\begin{equation*}
			C_{k,r}^{(t)}(\vec{v}) = C_{k,r}^{(t)}(\vec{w}) \ \Longrightarrow \ \hb_{\vec{v},k}^\tup{t} = \hb_{\vec{w},k}^\tup{t}.
		\end{equation*}}
\end{proposition}
Moreover, we can also show the converse, resulting in the following theorem.
\medskip
\begin{proposition}
	\label{thm:ho-lower}
	{\em	Let $G = (V(G), R_1(G), \dots, R_r(G), \ell)$ be a labeled, multi-relational graph.  Then for all $t \geq 0$ and $k \geq 1$, there exists $\UPD^\tup{t}$, $\AGG^\tup{t}$,  such that for all $\vec{v}$ and $\vec{w}$ in $V(G)^k$,
		\begin{equation*}
			C_{k,r}^{(t)}(\vec{v}) = C_{k,r}^{(t)}(\vec{w}) \ \Longleftrightarrow \ \hb_{\vec{v},k}^\tup{t} = \hb_{\vec{w},k}^\tup{t}.
		\end{equation*}
	}
\end{proposition}

The following result implies that increasing $k$ leads to a strict boost in terms of expressivity of the \localmkwl{} and \krn{k} architectures in terms of distinguishing non-isomorphic multi-relational graphs.
\medskip
\begin{proposition}\label{hier_neural}
	{\em For $k \geq 2 $ and $r \geq 1$, there exists a pair of non-isomorphic multi-relational graphs  $G_r = (V(G_r), R_1(G_r), \dots, R_r(G_r), \ell)$ and $H= (V(H_r), R_1(H_r), \dots, R_r(H_r), \ell)$ such that:
		\begin{itemize}
			\item For all choices of $\UPD^\tup{t}$, $\AGG^\tup{t}$, for $t > 0$, and $\RO$  the \rn{} architecture will not distinguish the graphs $G_r$ and $H_r$.
			\item There exists $\UPD^\tup{t}$, $\AGG^\tup{t}$, for $t > 0$, and $\RO$ such that the $(k+1)$\text{-}\textsf{RN} will distinguish them.
		\end{itemize}}
\end{proposition}

Moreover, the following results shows that for $k=2$ the \krn{k} architecture is strictly more expressive than \comp{} and \rgcn{} in distinguishing non-isomophics graphs.
\medskip
\begin{corollary}
	{\em There exists a \krn{2} architecture that is strictly more expressive than the \comp and the \rgcn\ architecture in terms of distinguishing non-isomorphic graphs. }
\end{corollary}

\paragraph{\krns{k} for Vertex-level Prediction}
As defined in Equations~\ref{kgnn} and~\ref{readout_k}, an \krn{k} architecture either computes $k$-tuple- or graph-level features. However, it is straightforward to compute a vertex-level features, see, e.g.,~\citet[Section 4.1]{Mor+2022a}.

\paragraph{Scalability} Although the \krn{k} is provably expressive, see Proposition~\ref{hier_neural}, it suffer some high memory requirement. Similar to the \kwl, it's memory complexity can only be lower bounded in~$\Omega(n^k)$, making it not applicable for large knowledge graphs. However, recent progress in making higher-order architectures more scalable, e.g.,~\cite{Bev+2021,Mor+2022a,Qia+2022}, can be straightforwardly lifted to the multi-relational case.

\section{Experimental Study}
\label{sec:experiments}

Here, we investigate to what extend the above theoretical results hold for real-world data distributions. Specifically, we aim to answer the following questions.

\begin{description}
	\item[Q1] Does the theoretical equivalence of \rgcn{} and \comp{} hold in practice?\\
	\item[Q2] Does the performance depend on the dimension of vertex features?\\
	\item[Q3] Does \comp{} benefit from normalization and learnable edge weights?\\
	\item[Q4] Does the theoretical difference in composition functions of \comp{} hold in practice?\\
\end{description}

\paragraph{Datasets.} To answer Q1 to Q4, we investige \rgcn\ and \comp's empirical performance on the small-scale AIFB (6\,000 vertices) and the large-scale AM (1.6 million vertices)~\citep{ristoski2016collection} vertex classification benchmark dataset; see Section~\ref{app:data} for dataset statistics.

\paragraph{Featurization.}
Most relational GNNs for vertex- and link-level tasks assume that the initial vertex states come from a learnable vertex embedding matrix~\citep{wang2021survey, ali2021bringing}. However, this vertex feature initialization or featurization method makes the model inherently transductive, i.e., the model must be re-trained when adding new vertices. Moreover, such an initialization strategy is incompatible with our Weisfeiler-Leman-based theoretical results since a learnable vertex embedding matrix will result in most initial vertex features being pair-wise different. Here, however, being faithful to the Weisfeiler-Leman formulation, we initialize \emph{all} vertex features with the \emph{same} $d$-dimensional vector, namely, a standard basis vector of $\bbR^{d}$, e.g., $(1, 0, \dots, 0)$ in $\bbR^{d}$.\footnote{We also probed a vector initialized with the \citet{glorot2010understanding} strategy, showing similar results.}
Relation-specific weight matrices in the case of \rgcn{} and edge features in the case of \comp{} are still learnable. We stress here that such a featurization strategy endows GNNs with inductive properties. Since we are using the same vertex feature initialization, we can run inference on previously unseen vertices or graphs.

\begin{figure}[t]
	\centering
	\begin{subfigure}[b]{0.60\textwidth}
		\centering
		\includegraphics[width=0.8\textwidth]{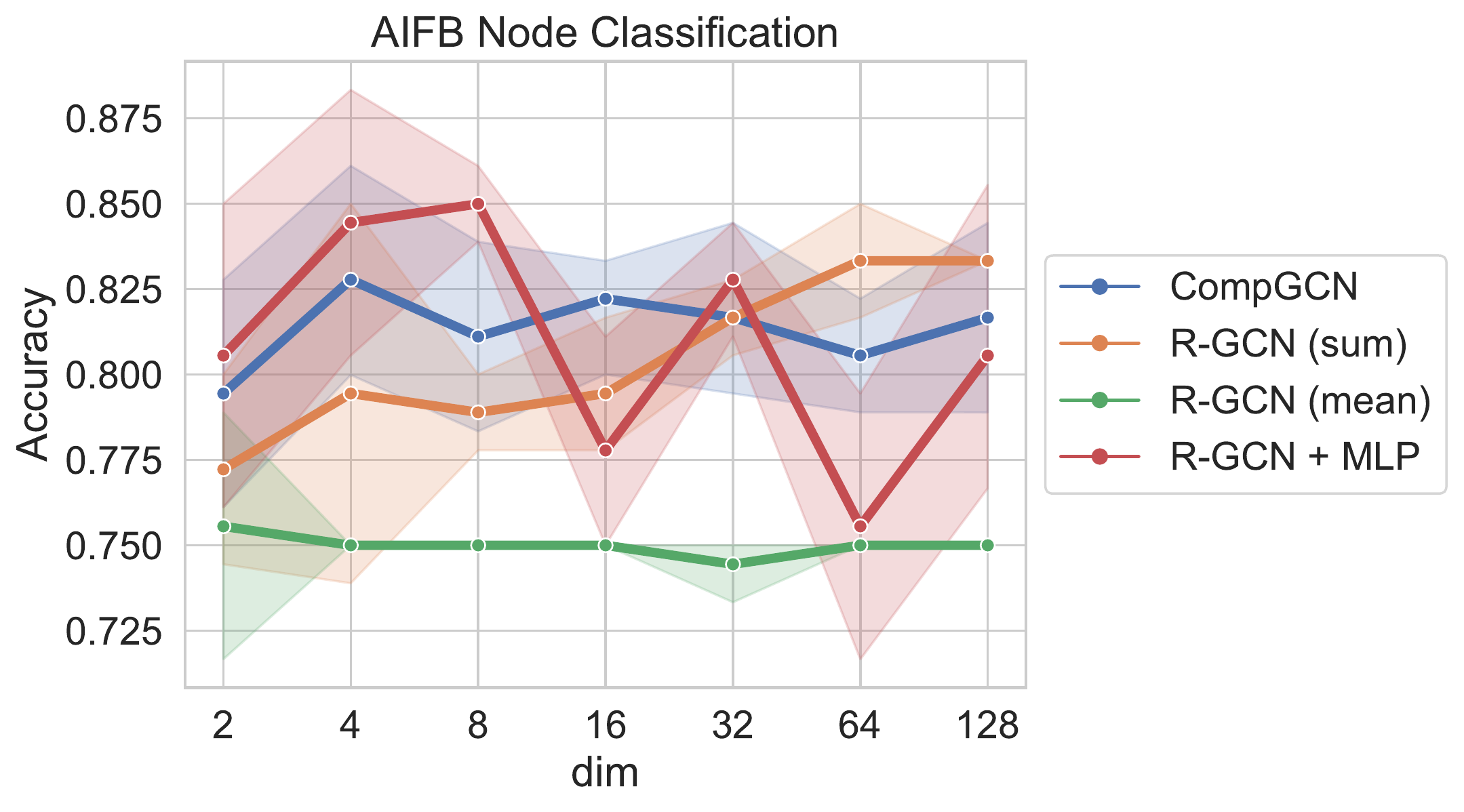}
		\caption{AIFB results with varying input feature dimension.}
		\label{fig:exp1_1}
	\end{subfigure}
	\begin{subfigure}[b]{0.35\textwidth}
		\centering
		\includegraphics[width=\textwidth]{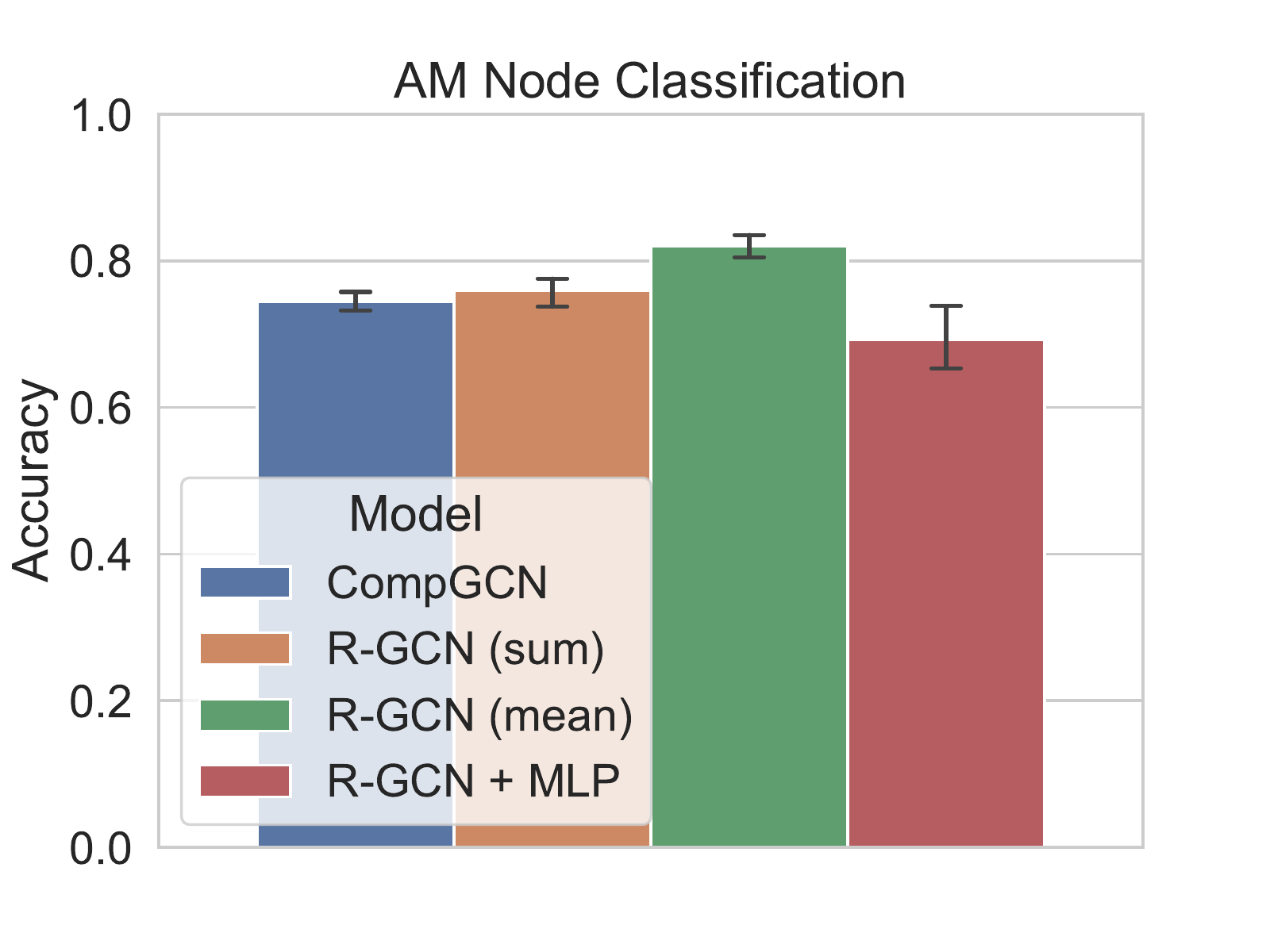}
		\caption{AM results with $dim=4$.}
		\label{fig:exp1_2}
		\vspace{0pt}
	\end{subfigure}
	\caption{Vertex classification performance of \comp and \rgcn\ on smaller (AIFB) and larger (AM) graphs. Initial vertex feature dimensions higher than 4 do not improve the accuracy.}
\end{figure}

\paragraph{Implementation.}
We use the \rgcn\ and \comp implementation provided by PyG framework~\citep{Fey+2019}. The source code of all methods and evaluation procedures is available at \url{https://github.com/migalkin/RWL}. For the smaller AIFB dataset, both models use two GNN layers. For the larger AM dataset, \rgcn\ saturates with three layers. Following the theory, we do not use any basis decomposition of relation weights in \rgcn. We list other hyperparameters in Section~\ref{app:data}.
We report averaged results of five independent runs using different random seeds. We conducted all experiments in the full-batch mode on a single GPU (Tesla V100 32\,GB or RTX 8000).

\paragraph{Discussion.}
Probing \rgcn\ with different aggregations and \comp on the smaller AIFB (Figure~\ref{fig:exp1_1}) and larger AM (Figure~\ref{fig:exp1_2}) datasets, we largely confirm the theoretical hypothesis of their expressiveness equivalence (\textbf{Q1}) and observe similar performance of both GNNs. The higher variance on AIFB is due to the small test set size (36 vertices), i.e., one misclassified vertex drops accuracy by $\approx 3\%$.

To test if increasing the input vertex feature dimensions leads to more expressive GNN architectures (\textbf{Q2}), we vary the initial vertex feature dimension in  $\{2,4,8,\ldots, 64,128\}$ on the smaller AIFB dataset (Figure~\ref{fig:exp1_1}) and do not observe any significant differences starting from $d=4$ and above. Having identified that, we report the best results of compared models on the larger AM graph with the vertex feature dimension $d$ in $\{4,8\}$.

Following the theory where the \texttt{sum} aggregator is most expressive, we investigate this finding on the smaller AIFB dataset for both GNNs. \rgcn\ with \texttt{mean} aggregation shows slightly better results on the larger AM dataset, which we attribute to the unstable optimization process of the \texttt{sum} aggregator where vertices might have thousands of neighbors, leading to large losses and noisy gradients. We hypothesize that stabilizing the training process on larger graphs might improve performance.

Furthermore, we perform an ablation study (Figure~\ref{fig:exp2}) of main \comp components (\textbf{Q3}), i.e., direction-based weighting (over direct, inverse, and self-loop edges), relation projection update in each layer, and message normalization in the GCN style  $\vec{D}^{-\frac{1}{2}}\vec{A}\vec{D}^{-\frac{1}{2}}$; see also Sections~\ref{APP:comp} and~\ref{app:rgcn}.

The crucial components for the smaller and larger graphs are (1) three-way direction-based message passing and (2) normalization. Replacing message passing over three directions (and three weight matrices) with one weight matrix using a single adjacency leads to a significant drop in performance. Removing normalization increases variance in the larger graph. Finally, removing both directionality and normalization leads to significant degradation in predictive performance.

\begin{figure}[t]
	\centering
	\includegraphics[width=0.65\textwidth]{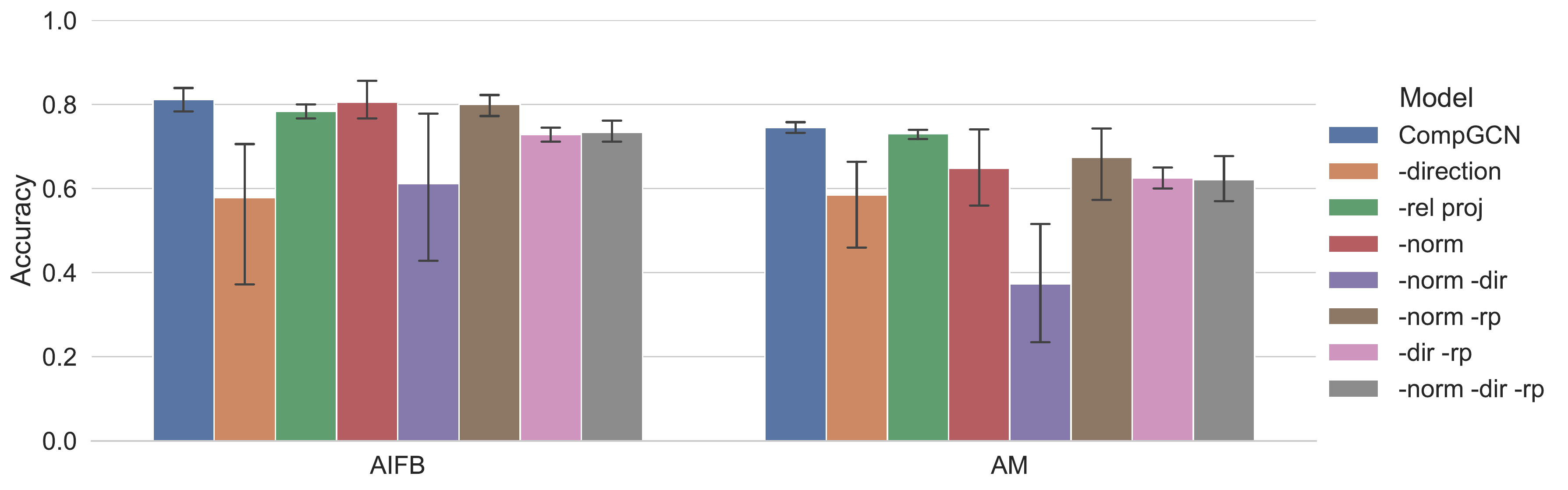}
	\caption{\comp ablations. Directionality (\emph{-dir}) and normalization (\emph{-norm}) are the most crucial components, i.e., their removal does lead to significant performance drops.}
	\label{fig:exp2}
\end{figure}

Studying composition functions (Figure~\ref{fig:exp3}), we do not find significant differences among non-parametric \texttt{mult}, \texttt{add}, \texttt{rotate} functions (\textbf{Q4}); see Section~\ref{APP:comp}.
Performance of an MLP over a concatenation of vertex and edge features 
falls within confidence intervals of other compositions and does not exhibit a significant accuracy boost.

\begin{figure}[t]
	\centering
	\includegraphics[width=0.55\textwidth]{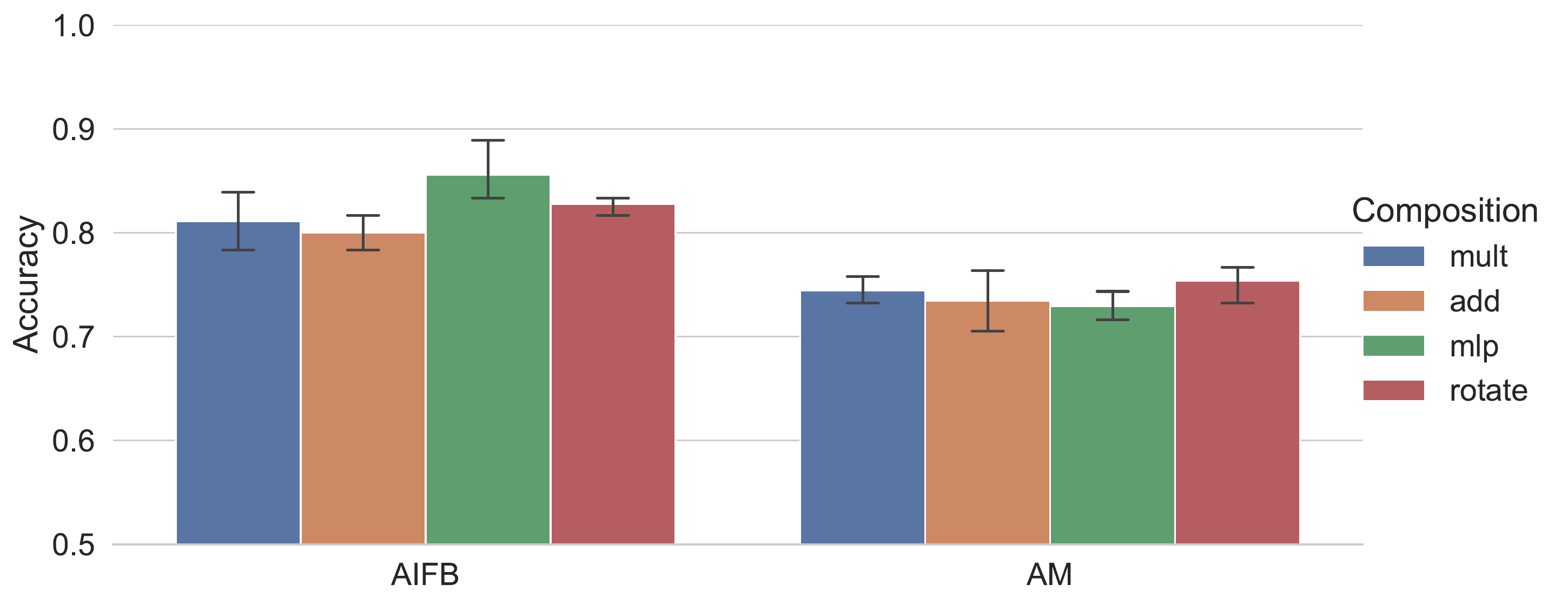}
	\caption{\comp with different composition functions. No significant differences.}
	\label{fig:exp3}
\end{figure}

\section{Conclusion}

Here, we investigated the expressive power of two popular GNN architectures for knowledge or multi-relational graphs, namely, \comp{} and \rgcn. By deriving a variant of the \wlone, we quantified their limits in distinguishing vertices in multi-relational graphs. Further, we investigated under which conditions,  i.e., the choice of the composition function, \comp{}, reaches the same expressive power as \rgcn. To overcome the limitations of the two architectures, we derived the provably more powerful \krn{k} architecture. By increasing $k$,  the \krn{k} architecture gets strictly more expressive. Empirically, we verified that our theoretical results translate largely into practice. Using \comp{} and \rgcn{} in a vertex classification setting over small and large multi-relational graphs shows that both architectures provide a similar performance level. We believe that our paper is the first step in a principled design of GNNs for knowledge or multi-relational graphs.

\section*{Acknowledgements}
Pablo Barcel\'o is funded by Fondecyt grant 1200967, the ANID - Millennium Science Initiative Program - Code ICN17002, and the National Center for Artificial Intelligence CENIA FB210017, Basal ANID. Mikhail Galkin is funded by the Samsung AI grant held at Mila. Christopher Morris is partially funded a DFG Emmy Noether grant (468502433) and  RWTH Junior Principal Investigator Fellowship under the Excellence Strategy of the Federal Government and the Länder. Miguel Romero is funded by Fondecyt grant 11200956, the Data Observatory Foundation, and the National Center for Artificial Intelligence CENIA FB210017, Basal ANID.

\bibliographystyle{unsrtnat}
\bibliography{bibliography}

\clearpage
\appendix

\section{Appendix}

\subsection{Related Work}\label{related_work}

In the following, we expand on relevant related work.

\paragraph{GNNs} Recently, GNNs~\citep{Gil+2017,Sca+2009} emerged as the most prominent graph representation learning architecture. Notable instances of this architecture include, e.g.,~\citet{Duv+2015,Ham+2017} and~\citet{Vel+2018}, which can be subsumed under the message-passing framework introduced in~\citet{Gil+2017}. In parallel, approaches based on spectral information were introduced in, e.g.,~\citet{Defferrard2016,Bru+2014,Kip+2017} and~\citet{Mon+2017}---all of which descend from early work in~\citet{bas+1997,Kir+1995,mic+2005,Mer+2005,mic+2009,Sca+2009} and~\citet{Spe+1997}.

\paragraph{Limits of GNNs and More Expressive Architectures}
Recently, connections between GNNs and Weisfeiler--Leman type algorithms have been shown~\citep{Mor+2019,Xu+2018b}. Specifically,~\citet{Mor+2019} and~\citet{Xu+2018b} showed that the $\wlone$ limits the expressive power of any possible GNN architecture in terms of distinguishing non-isomorphic graphs. In turn, these results have been generalized to the \kwl, see, e.g.,~\citet{Azi+2020,Gee+2020a,Gee+2020b,Mar+2019,Mor+2019,Morris2020b,Mor+2022b}, and connected to permutation-equivariant function approximation over graphs, see, e.g.,~\citet{Che+2019,geerts2022,Mae+2019}. \citet{barcelo2020} further
established an equivalence between the expressiveness of GNNs with readout functions
and $\textsf{C}^2$, the $2$-variable fragment of first-order logic extended by counting quantifiers.

\paragraph{Relational GNNs} Relational GNNs enjoy a profound usage in many areas of machine learning, such as complex question answering in NLP~\citep{fu2020survey} or visual question answering~\citep{huang2022endowing} in the intersection of NLP and vision. Notably, \citet{Sch+2019} proposed the first architecture, namely, \rgcn, being able to handle multi-relational data. Further, \citet{Vas+2020} proposed an alternative GNN architecture, namely, \comp, using less number of parameters and reporting improved empirical performance. In the knowledge graph reasoning area, \rgcn{} and \comp, being strong baselines, spun off numerous improved GNNs for vertex classification and transductive link prediction tasks~\citep{galkin2020message,yu2020generalized,compgcn_random_proj}. Furthermore, they inspired architectures for more complex reasoning tasks such as inductive link prediction~\citep{teru2020inductive,ali2021improving,Zhu+2021,zhang2022knowledge} and logical query answering~\citep{daza2020message,starqe,gnn_qe}.

Despite various applications, there has not been any theoretical work shedding light on multi-relational GNNs' expressive power and learning performance. Some recent empirical results highlight interesting properties of relational GNNs, e.g., a randomly initialized and untrained \rgcn{} still demonstrates non-trivial performance~\citep{degraeve2022r}, or that random perturbation of the relations does not lead to performance drops for \comp~\citep{compgcn_random_proj}.

\subsection{The Weisfeiler--Leman Algorithm}\label{APP:vr_ext}

In the following, we briefly describe Weisfeiler--Leman-type algorithms, starting with the \new{$1$-dimensional Weisfeiler--Leman algorithm} (\wlone{}).

\paragraph{The \wlone} The \wlone, or \new{color refinement}, is a simple heuristic for the graph isomorphism problem, originally proposed by~\citet{Wei+1968}.\footnote{Strictly speaking, \wlone\ and color refinement are two different algorithms. That is, \wlone\ considers neighbors and non-neighbors to update the coloring, resulting in a slightly higher expressive power when distinguishing vertices in a given graph, see~\citet{Gro+2021} for details. For brevity, we consider both algorithms to be equivalent.}
Intuitively, the algorithm determines if two graphs are non-isomorphic by iteratively coloring or labeling vertices. Given an initial coloring or labeling of the vertices of both
graphs, e.g., their degree or application-specific information, in
each iteration, two vertices with the same label get different labels if the number of identically labeled neighbors is not equal. If, after some iteration, the number of vertices annotated with a specific label is different in both graphs, the algorithm terminates and a stable coloring (partition) is obtained. We can then conclude that the two graphs are not isomorphic. It is easy to see that the algorithm cannot distinguish all non-isomorphic graphs~\citep{Cai+1992}. Nonetheless, it is a powerful heuristic that can successfully test isomorphism for a broad class of graphs~\citep{Arv+2015,Bab+1979,Kie+2015}.

Formally, let $G = (V(G),E(G),\ell)$ be a labeled graph. In each iteration, $t > 0$, the $\wlone$ computes a vertex coloring $C^{(t)} \colon V(G) \to \bbN$,
which depends on the coloring of the neighbors. That is, in iteration $t>0$, we set
\begin{equation*}
	C^{(t)}(v) \coloneqq \REL\Big(\!\big(C^{(t-1)}(v),\oms C^{(t-1)}(u) \mid u \in N(v)  \cms \big)\! \Big),
\end{equation*}
where $\REL$ injectively maps the above pair to a unique natural number, which has not been used in previous iterations.  In iteration $0$, the coloring $C^{(0)}\coloneqq \ell$. To test if two graphs $G$ and $H$ are non-isomorphic, we run the above algorithm in ``parallel'' on both graphs. If the two graphs have a different number of vertices colored $c$ in $\bbN$ at some iteration, the \wlone{} \new{distinguishes} the graphs as non-isomorphic. Moreover, if the number of colors between two iterations, $t$ and $(t+1)$, does not change, i.e., the cardinalities of the images of $C^{(t)}$ and $C^{(t+1)}$ are equal, or, equivalently,
\begin{equation*}
	C^{(t)}(v) = C^{(t)}(w) \iff C^{(t+1)}(v) = C^{(t+1)}(w),
\end{equation*}
for all vertices $v$ and $w$ in $V(G)$, the algorithm terminates. For such $t$, we define the \new{stable coloring}
$C^{\infty}(v) = C^{(t)}(v)$ for $v$ in $V(G)$. The stable coloring is reached after at most $\max \{ |V(G)|,|V(H)| \}$ iterations~\citep{Gro2017}.

Due to the shortcomings of the $\wlone$ or color refinement in distinguishing non-isomorphic
graphs, several researchers, e.g.,~\citep{Bab1979,Imm+1990}, devised a more powerful generalization of the former, today known
as the $k$-dimensional Weisfeiler-Leman algorithm (\kwl).\footnote{There exists two definitions of the \kwl, the so-called oblivious \kwl\ and the folklore or non-oblivious \kwl, see~\citet{Gro+2021}. There is a subtle difference in how they aggregate neighborhood information. Within the graph learning community, it is customary to abbreviate the oblivious \kwl\ as \kwl, a convention that we follow in this paper.}

\paragraph{Oblivious \kwl}  Intuitively, to surpass the limitations of the $\wlone$, the \kwl{} colors ordered subgraphs instead of a single vertex. More precisely, given a graph $G$, it colors the tuples from $V(G)^k$ for $k \geq 2$ instead of the vertices. By defining a neighborhood between these tuples, we can define a coloring  similar to the $\wlone$. Formally, let $G$ be a labeled graph, and let
$k \geq 2$. In each iteration $t \geq 0$, the algorithm, similarly to the $\wlone$, computes a
\new{coloring} $C^{(t)}_k \colon V(G)^k \to \bbN$. In the first iteration, $t=0$, the tuples $\vec{v}$ and $\vec{w}$ in $V(G)^k$ get the same
color if they have the same atomic type, i.e.,
$C^{(0)}_k(\vec{v}) \coloneqq \mathsf{atp}(\vec{v})$. Now, for $t \geq 0$, $C^{(t+1)}_k$ is defined
by
\begin{equation*}\label{APP:ci}
	C_{(t+1)}^k (\vec{v}) \coloneqq \REL \Big(\!\big( C^{(t)}_k(\vec{v}), M^{(t)}(\vec{v}) \big)\! \Big),
\end{equation*}
with $M^{(t)}(\vec{v})$ the tuple
\begin{equation}\label{APP:mi}
	M^{(t)}(\vec{v}) \coloneqq  \big( \{\!\! \{  C^{(t)}_k(\theta_1(\vec{v},w)) \mid w \in V(G) \} \!\!\}, \dots, \{\!\! \{  C^{(t)}_k(\theta_k(\vec{v},w)) \mid w \in V(G) \} \!\!\} \big).
\end{equation}
We also call $M^{(t)}$ an \new{aggregation function}. Here
\begin{equation*}
	\theta_j(\vec{v},w)\coloneqq (v_1, \dots, v_{j-1}, w, v_{j+1}, \dots, v_k).
\end{equation*}
That is, $\theta_j(\vec{v},w)$ replaces the $j$-th component of the tuple $\vec{v}$ with the vertex $w$. Hence, two tuples $\vec{v}$ and $\vec{w}$ with the same color in iteration $t$ get different colors in iteration $(t+1)$ if there exists a $j$ in $[k]$ such that the number of $j$-neighbors of $\vec{v}$ and $\vec{w}$, respectively, colored with a certain color is different.

Hence, two tuples are \new{adjacent} or \new{$j$-neighbors} if they are different in the $j$th component (or equal, in the case of self-loops). Again, we run the algorithm until convergence, i.e.,
\begin{equation*}
	C^{(t)}_k(\vec{v}) = C^{(t)}_k(\vec{w}) \iff C^{(t+1)}_k(\vec{v}) = C^{(t+1)}_k(\vec{w}),
\end{equation*}
for all $\vec{v}$ and $\vec{w}$ in $V(G)^k$ holds, and call the partition of $V(G)^k$
induced by $C^{(t)}_k$ the stable partition. For such $t$, we define
$C^{\infty}_k(\vec{v}) \coloneqq C^{(t)}_k(\vec{v})$ for $\vec{v}$ in $V(G)^k$.

To test whether two graphs $G$ and $H$ are non-isomorphic, we run the \kwl{} in ``parallel'' on both graphs. Then, if the two graphs have a different number of $k$-tuples colored $c$ in $\bbN$, the \kwl{} \textit{ distinguishes} the graphs as non-isomorphic. By increasing $k$, the algorithm becomes more powerful in distinguishing non-isomorphic graphs, i.e., for each $k \geq 1$, there are non-isomorphic graphs distinguished by $(k+1)$\text{-}\textsf{WL} but not by \kwl~\citep{Cai+1992}.

\paragraph{Local $\boldsymbol{\delta}$-$\boldsymbol{k}$-dimensional Weisfeiler--Leman Algorithm}\label{APP:lwl_ext}

\citet{Morris2020b} introduced a more efficient variant of the \kwl, the \new{local $\delta$-$k$-dimensional Weisfeiler--Leman algorithm} (\localkwl). In contrast to the \kwl, the  \localkwl{} considers only a subset of the entire neighborhood of a vertex tuple. Let the tuple $\vec{w} = \theta_j(\vec{v},w)$ be a $j$-{neighbor} of $\vec{v}$. We say that $\vec{w}$ is a \new{local} $j$-neighbor of $\vec{v}$ if $w$ is adjacent to the replaced vertex $v_j$. Otherwise, the tuple $\vec{w}$ is a \new{global} $j$-neighbor of $\vec{v}$. The \localkwl{} considers only local neighbors during the neighborhood aggregation process, and discards any information about the global neighbors. Formally, the \localkwl{} refines a coloring $C^{(t)}_{k,\delta}$ (obtained after $t$ rounds of the \localkwl) via the aggregation function
\begin{equation*}\label{APP:eqnmidd_ext}
	\begin{split}
		M^{(t)}_{\delta}(\vec{v}) \coloneqq  \big( \{\!\! \{ C^{(t)}_{k,\delta}(\theta_1(\vec{v},w)) \mid w \in N(v_1) \} \!\!\}, \dots, \{\!\! \{  C^{(t)}_{k,\delta}(\theta_k(\vec{v},w)) \mid w \in N(v_k) \}  \!\!\} \big),
	\end{split}
\end{equation*}
hence considering only the local $j$-neighbors of the tuple $\vec{v}$ in each iteration.  The coloring functions for the iterations of the \localkwl{} are then defined by
\begin{equation*}\label{APP:wlsimple_ext}
	C^{(t+1)}_{k,\delta}(\vec{v}) \coloneqq \REL \Big(\!\big( C^{(t)}_{k,\delta}(\vec{v}), M^{(t)}_{\delta}(\vec{v}) \big)\! \Big).
\end{equation*}
Note that the \wlone{} is equivalent to the $\delta$-1-\textsf{LWL}. \citet{Morris2020b} showed that, for each $k$, the \localkwl{} can distinguish graphs that the \kwl{} cannot and derived a variation of the former that is strictly more powerful than the \kwl.

\section{Missing Proofs in Section~\ref{sec:multi-wl}}

\begin{theorem}[Theorem~\ref{thm:mrgcn-upper} in the main text]
	\label{APP:thm:mrgcn-upper}
	{\em Let $G = (V(G), R_1(G), \dots, R_r(G), \ell)$ be a labeled, multi-relational graph. Then for all $t \geq 0$ the following hold:
		\begin{itemize}
			\item For all choices of initial vertex features consistent with $\ell$, sequences $\mathbf{W}_\rgcn^{(t)}$ of \rgcn\ parameters,
			      and vertices $v$ and $w$ in $V(G)$,
			      \begin{equation*}
				      C^{(t)}_{\textsf{R}}(v) =  	C^{(t)}_{\textsf{R}}(w) \ \ \Longrightarrow \ \   \hb_{v, \rgcn}^\tup{t} = \hb_{w, \rgcn}^\tup{t}.
			      \end{equation*}
			\item For all choices of initial vertex features consistent with $\ell$, sequences $\mathbf{W}_\comp^{(t)}$ of \comp
			      parameters,
			      composition functions $\phi$, and vertices $v$ and $w$ in $V(G)$,
			      \begin{equation*}
				      C^{(t)}_{\textsf{R}}(v) =  	C^{(t)}_{\textsf{R}}(w) \ \ \Longrightarrow \ \  \hb_{v, \comp}^\tup{t} = \hb_{w, \comp}^\tup{t}.
			      \end{equation*}
		\end{itemize} }
\end{theorem}

\begin{proof}
	We only prove it for \comp as the proof for \rgcn\ is analogous.
	Fix initial vertex features $(\hb_{v}^{(0)})_{v\in V(G)}$ for $G$ consistent with $\ell$, a sequence $\mathbf{W}_\comp^{(t)}$ of parameters,
	a composition function $\phi$, and two vertices $v$ and $w$ in $V(G)$.
	We prove the result by induction on $t \geq 0$.
	For $t = 0$, the statement follows immediately
	from the fact the initial features $(\hb_{v}^\tup{0})_{v\in V(G)}$ are consistent with $\ell$.
	Assume now that $C^{(t)}_{\textsf{R}}(v) = C^{(t)}_{\textsf{R}}(w)$, for $t>0$. Hence, by Equation~\ref{wlscomp},
	it must be the case
	that
	\begin{itemize}
		\item $C^{(t-1)}_{\textsf{R}}(v) =  	C^{(t-1)}_{\textsf{R}}(w)$, and
		\item $\oms C^{(t-1)}_{\textsf{R}}(u) \mid u \in\!N_i(v)  \cms = \oms C^{(t-1)}_{\textsf{R}}(u) \mid u \in\!N_i(w)  \cms$, for each $i \in [r]$.
	\end{itemize}
	Then, by induction hypothesis, it holds that:
	\begin{itemize}
		\item $\hb_{v, \comp}^\tup{t-1} = \hb_{w, \comp}^\tup{t-1}$, and
		\item $\oms \hb_{u, \comp}^\tup{t-1} \mid u \in\!N_i(v)  \cms = \oms \hb_{u, \comp}^\tup{t-1} \mid u \in\!N_i(w)  \cms$, for each $i$ in $[r]$.
	\end{itemize}
	From these two we conclude by applying
	Equation~\ref{eq:comp}
	that $\hb_{v, \comp}^\tup{t} = \hb_{w, \comp}^\tup{t}$.
	This is because we
	have that $\hb_{v, \comp}^\tup{t-1}  \mW^{(t)}_0 = \hb_{w, \comp}^\tup{t-1}   \mW^{(t)}_0$
	and
	$$\sum_{u \in N_i(v)}  \phi \big(\hb_{u, \comp}^\tup{t-1}, \zb_i^{(t)} \big)  \mW_1^{(t)} \ = \
		\sum_{u \in N_i(w)}  \phi \big(\hb_{u, \comp}^\tup{t-1}, \zb_i^{(t)} \big)  \mW_1^{(t)},
	$$
	for each $i \in [r]$.
\end{proof}
\medskip
\begin{theorem}[Theorem~\ref{thm:mrgcn-lower} in the main text]\label{APP:thm:mrgcn-lower}
	{\em Let $G=(V(G), R_1(G), \dots, R_r(G), \ell)$ be a labeled, multi-relational graph. For all $t\geq 0$:
		\begin{itemize}
			\item
			      There exist initial vertex features and a sequence $\mathbf{W}_\rgcn^{(t)}$ of parameters such that for all $v$ and $w$ in $V(G)$,
			      \begin{equation*}
				      C^{(t)}_{\textsf{R}}(v) =  	C^{(t)}_{\textsf{R}}(w) \ \ \Longleftrightarrow \ \   \hb_{v, \rgcn}^\tup{t} = \hb_{w, \rgcn}^\tup{t}.
			      \end{equation*}
			\item
			      There exist initial vertex features, a sequence $\mathbf{W}_\comp^{(t)}$ of parameters
			      and a composition function $\phi$ such that for all $v$ and $w$ in $V(G)$,
			      \begin{equation*}
				      C^{(t)}_{\textsf{R}}(v) =  	C^{(t)}_{\textsf{R}}(w) \ \ \Longleftrightarrow \ \   \hb_{v, \comp}^\tup{t} = \hb_{w, \comp}^\tup{t}.
			      \end{equation*}
		\end{itemize} }
\end{theorem}

\begin{proof}
	We focus on the case of \comp{} when the composition map $\phi$ is vector scaling, that is, $\phi(\hb, \alpha) = \alpha \hb$, for $\hb$ in $\mathbb{R}^d$ and $\alpha$ in $\mathbb{R}$. As we explain later, this implies the cases of \rgcn{}, \comp{} with point-wise multiplication, and also \comp{} with circular-correlation.

	The proof is a refinement of \cite[Theorem 2]{Mor+2019} for multi-relational graphs. For a matrix $\mB$, we denote by $\mB_i$ its $i$-th row. Let $n=|V(G)|$ and without loss of generality assume $V(G)=[n]$. We represent vertex features for $G$ as a matrix $\mF$ in $\mathbb{R}^{n\times d}$, where $\mF_v$ corresponds to the vertex feature of $v$. By slightly abusing notation, we view vertex features as a coloring for $G$. In particular, we denote by $\Gamma_G(\mF)$ the application of one step of the  \rwlone{} on $G$. That is,
	$\Gamma_G(\mF)$ is a coloring $C \colon V(G)\to \mathbb{N}$ such that for each $v$ in $V(G)$,
	$$ C(v) \coloneqq \RL \Big(\!\big(C_{\mF}(v), \oms (C_{\mF}(u), i) \mid i \in [r], u \in\!N_i(v) \cms \big)\! \Big), $$
	where $C_{\mF}$ is the coloring corresponding to the matrix $\mF$.
	On the other hand, the update rule of \comp{} can be written as follows:
	$$\mF' = \sigma(\mF \mW_0  + \sum_{i\in[r]} \alpha_i \mA_i \mF \mW_1 + b\mJ),$$
	where $\mW_0$ and $\mW_1$ are the parameter matrices, $\alpha_i$ are the scaling factors, $\mA_i$ is the adjacency matrix for the relation $R_i(G)$, and $\mJ$ is the all-one matrix of appropriate dimensions, representing the biases. Here we choose $\sigma$ to be the sign function $\sign$ and the bias $b$ to be $b=-1$. Using the same argument as in \cite[Corollary 16]{Mor+2019}, we can replace $\sigma$ by the ReLU function.

	We need the following lemma shown in \cite[Lemma 9]{Mor+2019}.

	\begin{lemma}[\cite{Mor+2019}]
		\label{APP:lemma:sign-matrix}
		{\em
		Let $\mB$ in $\mathbb{N}^{s\times t}$ be a matrix such that all the rows are pairwise distinct. Then there is a matrix $\mX$ in $\mathbb{R}^{t\times s}$ such that the matrix $\sign(\mB\mX-\mJ)$   in $\{-1,1\}^{s\times s}$ is non-singular.
		}
	\end{lemma}
	Following~\cite{Mor+2019}, we say that a matrix is \emph{row-independent modulo equality} if the set of all rows appearing in the matrix is linearly independent. For two colorings $C$ and $C'$ of $G$, we write $C\equiv C'$ if the colorings define the same partition on $V(G)$. The key lemma of the proof is the following:
	\begin{lemma}
		\label{APP:lemma:key-wl-comp}
		{\em
			Let $\mF$ in $\mathbb{R}^{n\times d}$ be row-independent modulo equality.
			Then there are matrices $\mW_0$ and $\mW_1$ in $\mathbb{R}^{d\times e}$ and scaling factors $\alpha_i$ in $\mathbb{R}$, for $i$ in $[r]$, such that the matrix
			$$ \mF' = \sign(\mF \mW_0  + \sum_{i\in[r]} \alpha_i \mA_i \mF \mW_1 - \mJ)$$
			is row-independent modulo equality and $\mF' \equiv \Gamma_G(\mF)$.
		}
	\end{lemma}

	\begin{proof}
		Let $q$ be the number of distinct rows in $\mF$ and let $\widetilde{\mF}$ in $\mathbb{R}^{q\times d}$ be the matrix whose rows are the distinct rows of $\mF$ in an arbitrary but fixed order. We denote by $Q_1,\dots,Q_q$ the associated \emph{color classes}, that is, a vertex $v$ in $[n]$ is in $Q_j$ if and only if $\mF_v = \widetilde{\mF}_j$. By construction, the rows of $\widetilde{\mF}$
		are linearly independent, and hence there is a matrix $\mM$ in $\mathbb{R}^{d\times q}$ such that $\widetilde{\mF}\mM$ in $\mathbb{R}^{q\times q}$ is the identity matrix. It follows that the matrix $\mF\mM$ in $\mathbb{R}^{n\times q}$ has entries:
		$$(\mF\mM)_{vj} =
			\begin{cases}
				1 & \text{if $v\in Q_j$} \\
				0 & \text{otherwise}.
			\end{cases}$$
		Let $\mD$ in $\mathbb{N}^{n\times q(r+1)}$ be the matrix with entries:
		$$\mD_{vh} =
			\begin{cases}
				|N_i(v)\cap Q_j| & \text{if $h = iq + j$ for $i\in [r]$, $j\in [q]$} \\
				1                & \text{if $h\in[q]$ and $v\in Q_h$}                \\
				0                & \text{otherwise}.
			\end{cases}$$
		So the $v$-th row of $\mD$ is the concatenation of a one-hot vector encoding of the color of $v$ and a vector encoding for the multiset of the colors in $N_i(v)$, for each $i$ in $[r]$. We have
		$$\Gamma_G(\mF) \equiv \mD$$
		if we view $\mD$ as a coloring of $G$. We can also see $\mD$ as a block matrix $\mD = [\mN_0\, \mN_1 \cdots \mN_r]$, where $\mN_0 = \mF\mM$ in $\mathbb{N}^{n\times q}$ and $\mN_i=\mA_i \mF \mM$ in $\mathbb{N}^{n\times q}$ for each $i$ in $[r]$. Since $0\leq \mD_{vh} \leq n-1$, for all $v$ in $[n], h$ in $[q(r+1)]$, we have
		$$\mD \equiv \mE$$
		where
		$$\mE = \mF\mM + \sum_{i\in [r]}n^i \mA_i\mF\mM.$$
		Indeed, $\mE_{vj}$ is simply the $n$-base representation of the vector $(\mD_{vj}, \mD_{v(qj)}, \dots, \mD_{v(rqj)})$, and hence $\mE_v = \mE_w$ if and only if $\mD_v = \mD_w$.

		Let $p$ be the number of distinct rows in $\mE$ and let $\widetilde{\mE}$ in $\mathbb{N}^{p\times q}$ be the matrix whose rows are the distinct rows of $\mE$ in an arbitrary but fixed order. We can apply  Lemma~\ref{APP:lemma:sign-matrix} to $\widetilde{\mE}$ and obtain a matrix $\mX$ in $\mathbb{R}^{q\times p}$ such that $\sign(\widetilde{\mE}\mX - \mJ)$ in $\mathbb{R}^{p\times p}$ is non-singular.
		In particular, $\sign(\mE\mX - \mJ)$ is row-independent modulo equality and
		$\sign(\mE\mX - \mJ)\equiv \mE \equiv \Gamma_{G}(\mF)$.
		Let $\mW_0 = \mW_1 = \mM\mX$ in $\mathbb{R}^{d\times p}$ and $\alpha_i =n^i$ for $i$ in $[r]$. We have
		\begin{align*}
			\mF' & = \sign(\mF \mW_0  + \sum_{i\in[r]} \alpha_i \mA_i \mF \mW_1 - \mJ)  \\
			     & = \sign(\mF\mM\mX  + \sum_{i\in[r]} \alpha_i \mA_i \mF \mM\mX - \mJ) \\
			     & = \sign(\mE\mX - \mJ).
		\end{align*}
		Hence $\mF'$ is row-independent modulo equality and $\mF'= \sign(\mE\mX - \mJ) \equiv \Gamma_{G}(\mF)$.
	\end{proof}

	Now the theorem follows directly from Lemma~\ref{APP:lemma:key-wl-comp}. We start with initial vertex features $(\hb_v^{(0)})_{v \in V(G)}$ consistent with $\ell$ such that different features are linearly independent. Hence the matrix $\mF^{(0)}$ representing the initial features is row-independent modulo equality and we can apply iteratively Lemma~\ref{APP:lemma:key-wl-comp} to obtain the required sequence $\mathbf{W}_\comp^{(t)}$ such that $C^{(t)}_{\textsf{R}} \equiv \mF^{(t)}$, where $\mF^{(t)}$ is the matrix representing the vertex features $(\hb_{v, \comp}^{(t)})_{v \in V(G)}$. In particular, $C^{(t)}_{\textsf{R}}(v) =  	C^{(t)}_{\textsf{R}}(w) \Leftrightarrow  \hb_{v, \comp}^\tup{t} = \hb_{w, \comp}^\tup{t}$, for all $v$ and $w$ in $V(G)$.

	\begin{remark}
		Note that the dimensions $d\times e$ of the parameter matrices at layer $t$ correspond to the number of distinct colors before ($q$) and after ($p$) the application of the layer.
	\end{remark}


	The case of \comp{} with point-wise multiplication holds since we can simulate vector scaling as
	$\alpha \hb = \hb * (\alpha,\dots,\alpha)$, where $*$ denotes point-wise multiplication. Similarly, the case of \rgcn{} follows as we can simulate vector scaling by setting $\mW_i = \alpha_i \mW_1$, for each $i$ in $[r]$.

	Finally, we show that the result also holds for \comp{} with circular correlation. This composition map is defined as follows\footnote{For $0$-indexed vectors, this is simply $(\hb \star \zb)_i = \sum_{j=0}^{d-1} \hb_j \zb_{(i+j)\, \text{mod}\, d}$ for $0\leq i \leq d-1$.}:
	$$(\hb \star \zb)_i = \sum_{j=1}^d \hb_j \zb_{((i+j-2)\, \text{mod}\, d) + 1 },$$
	where $\hb, \zb$ in $\mathbb{R}^d$, $\hb \star \zb$ in $\mathbb{R}^d$ and $i$ in $[d]$. We can easily simulate one layer of \comp{} with vector scaling using two layers of \comp{} with circular-correlation. Indeed, for a layer of the form
	$$\hb_{v} = \sigma \Big(\gb_{v} \mW_0 +
		\sum_{i \in [r]} \sum_{w \in N_i(v)}  \alpha_i \gb_{w}  \mW_1 + \bbb \Big),$$
	where $\gb_{u}$ in $\mathbb{R}^d$, for all $u$ in $V(G)$,
	we first use a layer of the form
	$$\widetilde{\hb}_v =  \gb_{v} \mP,$$
	where $\mP$ in $\mathbb{R}^{d\times d}$ reverts the vertex features, that is,
	all the entries are zero except for $\mP_{(n-i+1)i} = 1$ for all $i$ in $[d]$,
	followed by a layer
	$$\hb_{v} = \sigma \Big(\widetilde{\hb}_{v} \mP \mW_0 +
		\sum_{i \in [r]} \sum_{w \in N_i(v)}  (\widetilde{\hb}_{v}\star (0,\dots, 0, \alpha_i)) \mW_1 + \bbb \Big).$$
\end{proof}

\subsection{On the Choice of the Composition Function for \rgcn{} Architectures}

\begin{proposition}[Proposition~\ref{prop:weak-wl-sep} in the main text]
	\label{APP:prop:weak-wl-sep}
	{\em
		There exist a labeled, multi-relational graph $G=(V(G),R_1(G),R_2(G),\ell)$ and two vertices $v$ and $w$ in $V(G)$, such that $C^{(1)}_{\textsf{R}}(v)\neq C^{(1)}_{\textsf{R}}(w)$ but $C^{\infty}_{\textsf{WR}}(v) = C^{\infty}_{\textsf{WR}}(w)$.
	}
\end{proposition}

\begin{proof}
	We have $V(G) = \{v,w,u_1,u_2\}$, $R_1(G)=\{(v,u_1), (w, u_2)\}$, $R_2(G)=\{(v,u_2), (w,u_1)\}$, $\ell(v)=\ell(w)=0$, $\ell(u_1)=1$ and $\ell(u_2)=2$. Hence,
	$$
		C^{(1)}_{\textsf{R}}(v)  = \RL \Big(\!\big(0, \oms (1, 1), (2,2) \cms \big)\! \Big) \qquad
		C^{(1)}_{\textsf{R}}(w)  = \RL \Big(\!\big(0, \oms (2, 1), (1,2) \cms \big)\! \Big),
	$$
	that is, $C^{(1)}_{\textsf{R}}(v)\neq C^{(1)}_{\textsf{R}}(w)$. On the other hand,
	$$
		C^{(1)}_{\textsf{WR}}(v)  = \RL \Big(\!\big(0, \oms 1,2 \cms, 1, 1 \big)\! \Big) \qquad
		C^{(1)}_{\textsf{WR}}(w)  = \RL \Big(\!\big(0, \oms 1,2 \cms, 1, 1 \big)\! \Big)
	$$
	and then $C^{\infty}_{\textsf{WR}}(v) = C^{\infty}_{\textsf{WR}}(w)$.
\end{proof}


As shown next, the expressive power of \comp\ architectures that use point-wise summation/substraction
or vector concatenation is captured by this
weaker form of multi-relational $\wlone$.

\begin{theorem}[Theorem~\ref{thm:weak-comp-wl} in the main text] \label{APP:thm:weak-comp-wl}
	{\em
		Let $G = (V(G), R_1(G), \dots, R_r(G), \ell)$ be a labeled, multi-relational graph. Then:
		\begin{itemize}
			\item For all \mbox{$t\geq 0$}, choices of initial vertex features consistent with $\ell$,
			      sequence $\mathbf{W}_\comp^{(t)}$ of \comp\ parameters, and vertices
			      $v$ and $w$ in $V(G)$,
			      \begin{equation*}
				      C^{(t)}_{\textsf{WR}}(v) =  	C^{(t)}_{\textsf{WR}}(w) \ \ \Longrightarrow \ \  \hb_{v, \comp}^\tup{t} = \hb_{w, \comp}^\tup{t},
			      \end{equation*}
			      for either point-wise summation/substraction or concatenation as the composition map.

			\item For all \mbox{$t\geq 0$}, there exist initial vertex features and a sequence $\mathbf{W}_\comp^{(t)}$ of \comp\
			      parameters, such that for all vertices $v$ and $w$ in $V(G)$,
			      \begin{equation*}
				      C^{(t)}_{\textsf{WR}}(v) =  	C^{(t)}_{\textsf{WR}}(w) \ \ \Longleftrightarrow \ \   \hb_{v, \comp}^\tup{t} = \hb_{w, \comp}^\tup{t},
			      \end{equation*}
			      for either point-wise summation/substraction or concatenation as the composition map.
		\end{itemize}
	}
\end{theorem}

\begin{proof}
	We start with the first item.
	We focus first on the case of \comp{} with vector concatenation.
	Note that if $\hb$ in $\mathbb{R}^d$, $\zb$ in $\mathbb{R}^b$
	and $\mW$ in $\mathbb{R}^{(d+b)\times e}$, then we have
	$$(\hb, \zb) \mW = \hb \mX + \zb \mY,$$
	where $\mX$ in $\mathbb{R}^{d\times e}$ is the matrix given by the first $d$ rows of $\mW$, while $\mY$ in $\mathbb{R}^{b\times e}$ is the matrix given by the last $b$ rows of $\mW$. In particular, we can write
	\begin{align*}
		\hb_{v, \comp}^\tup{t} & = \sigma \Big(\hb_{v, \comp}^\tup{t-1}  \mW^{(t)}_0 +
		\sum_{i \in [r]} \sum_{u \in N_i(v)}  (\hb_{u, \comp}^\tup{t-1}, \zb_i^{(t)}) \mW_1^{(t)} \Big)              \\
		                       & = \sigma \Big(\hb_{v, \comp}^\tup{t-1}  \mW^{(t)}_0 +
		\sum_{i \in [r]} \sum_{u \in N_i(v)}  \hb_{u, \comp}^\tup{t-1}  \mX_1^{(t)} + \zb_i^{(t)}  \mY_1^{(t)} \Big) \\
		                       & = \sigma \Big(\hb_{v, \comp}^\tup{t-1}  \mW^{(t)}_0 +
		\sum_{i \in [r]} \sum_{u \in N_i(v)}  \hb_{u, \comp}^\tup{t-1}  \mX_1^{(t)} + 	\sum_{i \in [r]} |N_i(v)| \zb_i^{(t)}  \mY_1^{(t)} \Big).
	\end{align*}

	Fix initial vertex features $(\hb_{v}^{(0)})_{v\in V(G)}$ for $G$ consistent with $\ell$, a sequence $\mathbf{W}_\comp^{(t)}$ of parameters and two vertices $v$ and $w$ in $V(G)$.
	We proceed by induction on $t \geq 0$.
	For $t = 0$ we are done as the features $(\hb_{v}^{(0)})_{v\in V(G)}$ are consistent with $\ell$.
	Assume now that $C^{(t)}_{\textsf{WR}}(v) = C^{(t)}_{\textsf{WR}}(w)$, for $t>0$. Then, by Equation~\ref{eq:weak-wlscomp},
	we have that
	\begin{itemize}
		\item $C^{(t-1)}_{\textsf{WR}}(v) = C^{(t-1)}_{\textsf{WR}}(w)$,
		\item $\oms C^{(t-1)}_{\textsf{WR}}(u) \mid i\in[r], u \in\!N_i(v)  \cms = \oms C^{(t-1)}_{\textsf{WR}}(u) \mid i\in[r], u \in\!N_i(w)  \cms$,
		\item $|N_i(v)| = |N_i(w)|$ for each $i\in [r]$.
	\end{itemize}
	Then, by induction hypothesis, it holds that:
	\begin{itemize}
		\item $\hb_{v, \comp}^\tup{t-1} = \hb_{w, \comp}^\tup{t-1}$, and
		\item $\oms \hb_{u, \comp}^\tup{t-1} \mid i\in [r], u \in\!N_i(v)  \cms = \oms \hb_{u, \comp}^\tup{t-1} \mid i\in [r], u \in\!N_i(w)  \cms$.
	\end{itemize}
	Then we have
	\begin{itemize}
		\item $\sum_{i \in [r]} |N_i(v)| \zb_i^{(t)} = \sum_{i \in [r]} |N_i(w)| \zb_i^{(t)}$, and
		\item $\sum_{i \in [r]} \sum_{u \in N_i(v)}  \hb_{u, \comp}^\tup{t-1} =
			      \sum_{i \in [r]} \sum_{u \in N_i(w)}  \hb_{u, \comp}^\tup{t-1}$.
	\end{itemize}
	We conclude that $\hb_{v, \comp}^\tup{t} = \hb_{w, \comp}^\tup{t}$.

	\medskip

	Note that the update rule for the case of  point-wise summation/substraction is the same except that now $\mX_1^{(t)} =
		\mY_1^{(t)}$. Hence exactly the same argument applies.

	We now turn to the second item.
	We follow the same strategy and terminology as in the proof of Theorem \ref{APP:thm:mrgcn-lower}. In this case, given a vertex feature matrix $\mF$ in $\mathbb{R}^{n\times d}$, we denote by $\hat{\Gamma}_G(\mF)$ the application of one step of the weak \rwlone{}. Hence, $\hat{\Gamma}_G(\mF)$ is a coloring $C \colon V(G)\to \mathbb{N}$ such that for each $v$ in $V(G)$,
	$$ C(v) = \RL \Big(\!\big(C_{\mF}(v), \oms C_{\mF}(u) \mid i \in [r], u \in\!N_i(v) \cms, |N_1(v)|, \dots, |N_r(v)| \big)\! \Big), $$
	where $C_{\mF}$ is the coloring corresponding to the matrix $\mF$.
	In this case, the update rule for \comp{} with vector concatenation can be written as follows:
	$$\mF' = \sigma(\mF \mW_0  + \sum_{i\in[r]} \mA_i \mF \mX_1 + \sum_{i\in[r]} \mA_i \mZ_i \mY_1 + b\mJ),$$
	where $\mW_0$ in $\mathbb{R}^{d\times e}$ and $\mW_1=\begin{bmatrix} \mX_1\\ \mY_1\end{bmatrix}\in \mathbb{R}^{(d+b)\times e}$, for $\mX_1\in \mathbb{R}^{d\times e}$, $\mY_1\in \mathbb{R}^{b\times e}$, are the parameter matrices, $\mZ_i\in \mathbb{R}^{n\times b}$ is the matrix where each row is a copy of the edge feature $\zb_i\in \mathbb{R}^b$ associated with the relation $R_i(G)$, $\mA_i$ is the adjacency matrix for the relation $R_i(G)$, and $\mJ$ is the all-one matrix of appropriate dimensions. We have the following:

	\begin{lemma}
		\label{APP:lemma:key-weakwl}
		{\em
			Let $\mF$ in $\mathbb{R}^{n\times d}$ be row-independent modulo equality.
			Then there are matrices $\mW_0$ in $\mathbb{R}^{d\times e}$,  $\mX_1$ in $\mathbb{R}^{d\times e}$, $\mY_1$ in $\mathbb{R}^{b\times e}$ and vectors $\zb_i$ in $\mathbb{R}^b$, for $i$ in $[r]$ such that the matrix
			$$ \mF' = \sign(\mF \mW_0  + \sum_{i\in[r]} \mA_i \mF \mX_1 + \sum_{i\in[r]} \mA_i \mZ_i \mY_1 - \mJ)$$
			is row-independent modulo equality and $\mF' \equiv \hat{\Gamma}_G(\mF)$.
		}
	\end{lemma}
	\begin{proof}
		Let $q$ be the number of distinct rows in $\mF$ and let $\widetilde{\mF}$ in $\mathbb{R}^{q\times d}$ be the matrix whose rows are the distinct rows of $\mF$ in an arbitrary but fixed order. We denote by $Q_1,\dots,Q_q$ the associated \emph{color classes}, that is, a vertex $v$ in $[n]$ is in $Q_j$ if and only if $\mF_v = \widetilde{\mF}_j$. By construction, the rows of $\widetilde{\mF}$
		are linearly independent, and hence there is a matrix $\mM$ in $\mathbb{R}^{d\times q}$ such that $\widetilde{\mF}\mM$ in $\mathbb{R}^{q\times q}$ is the identity matrix. It follows that the matrix $\mF\mM$ in $ \mathbb{R}^{n\times q}$ has entries:
		$$(\mF\mM)_{vj} =
			\begin{cases}
				1 & \text{if $v\in Q_j$} \\
				0 & \text{otherwise}.
			\end{cases}$$
		Let $\mM_0,\mM_1$ in $ \mathbb{N}^{d\times (2q+r)}$, $\mM_2$ in $ \mathbb{N}^{r\times (2q+r)}$ be the block matrices
		$\mM_0 =[\mM \, \mO\, \mO']$, $\mM_1 =[\mO\, \mM\, \mO']$ and $\mM_2 = [\mO''\, \mO''\, \mI]$,
		where $\mO$ in $\mathbb{R}^{d\times q}$, $\mO'$ in $\mathbb{R}^{d\times r}$, $\mO''$ in $\mathbb{R}^{r\times q}$ are all-$0$ matrices, and $\mI$ in $\mathbb{R}^{r\times r}$ is the identity matrix. For each $i$ in $ [r]$, the required $\zb_i$ in $\mathbb{R}^r$ is the vector with all entries $0$ except for the $i$-th position which is $1$. Let $\mZ_i$ be the corresponding matrix whose rows are copies of $\zb_i$. We define $\mD$ in $ \mathbb{N}^{n\times (2q+r)}$ as:
		\begin{align*}
			\mD & = \mF\mM_0 + \sum_{i\in[r]} \mA_i \mF \mM_1 + \sum_{i\in[r]} \mA_i \mZ_i \mM_2                                                              \\
			    & = \begin{bmatrix}\mF\mM & \sum_{i\in[r]} \mA_i \mF \mM & \sum_{i\in[r]} \mA_i \mZ_i\end{bmatrix}.
		\end{align*}
		The $v$-th row of $\mF\mM$ encodes the color of $v$, the
		$v$-th row of $\sum_{i\in[r]} \mA_i \mF \mM$ encodes the multiset of the colors of $u$, when we range over $i$ in $[r]$ and $u$ in $ N_i(v)$, and the $v$-th row
		of $\sum_{i\in[r]} \mA_i \mZ_i$ contains the sizes of $N_i(v)$ for all $i$ in $[r]$. Hence,
		$$\hat{\Gamma}_G(\mF) \equiv \mD$$
		if we view $\mD$ as a coloring of $G$.

		Let $p$ be the number of distinct rows in $\mD$ and let $\widetilde{\mD}$ in $\mathbb{N}^{p\times (2q+r)}$ be the matrix whose rows are the distinct rows of $\mD$ in an arbitrary but fixed order. We apply  Lemma~\ref{APP:lemma:sign-matrix} to $\widetilde{\mD}$ and obtain a matrix $\mX$ in $ \mathbb{R}^{(2q+r)\times p}$ such that $\sign(\widetilde{\mD}\mX - \mJ)$ in $ \mathbb{R}^{p\times p}$ is non-singular.
		In particular, $\sign(\mD\mX - \mJ)$ is row-independent modulo equality and
		$\sign(\mD\mX - \mJ)\equiv \mD \equiv \hat{\Gamma}_{G}(\mF)$.
		Let $\mW_0 = \mM_0 \mX$ in $ \mathbb{R}^{d\times p}$, $\mX_1 = \mM_1 \mX$ in $ \mathbb{R}^{d\times p}$, and $\mY_1 = \mM_2 \mX$ in $ \mathbb{R}^{r\times p}$. We have
		\begin{align*}
			\mF' & =  \sign(\mF \mW_0  + \sum_{i\in[r]} \mA_i \mF \mX_1 + \sum_{i\in[r]} \mA_i \mZ_i \mY_1 - \mJ)             \\
			     & =  \sign(\mF \mM_0 \mX  + \sum_{i\in[r]} \mA_i \mF \mM_1 \mX + \sum_{i\in[r]} \mA_i \mZ_i \mM_2 \mX - \mJ) \\
			     & =  \sign(\mD\mX - \mJ).
		\end{align*}
		Hence $\mF'$ is row-independent modulo equality and $\mF'=  \sign(\mD\mX - \mJ) \equiv \hat{\Gamma}_{G}(\mF)$.
	\end{proof}

	The theorem follows directly by iteratively applying Lemma~\ref{APP:lemma:key-weakwl} starting with vertex features $(\hb^{(0)}_
		v)_{v\in V(G)}$ consistent with $\ell$ such that different features are linearly independent.

	\medskip

	The case of \comp{} with point-wise summation/substraction follows from the fact that this architecture can simulate \comp{} with vector concatenation. Indeed, we can simulate one layer of  \comp{} with vector concatenation
	using two layers of \comp{} with point-wise summation/substraction.
	Take a layer of the form
	$$\hb_{v} = \sigma \Big({\gb}_{v} \mW_0 +
		\sum_{i \in [r]} \sum_{w \in N_i(v)}  ({\gb}_{w}, \zb_i)  \mW_1 + {\bbb} \Big),$$
	where ${\gb}_{u}$ in $\mathbb{R}^d$, for $u$ in $ V(G)$, $\mW_0\in \mathbb{R}^{d\times e}$,
	$\mW_1\in \mathbb{R}^{(d+b)\times e}$ and $\zb_i\in\mathbb{R}^b$.
	We first use a layer
	$$\widetilde{\hb}_v =  {\gb}_{v} \mB,$$
	where $\mB\in \mathbb{R}^{d\times (d+b)}$ is the $d\times d$ identity matrix with $b$ additional all-$0$ columns. So $\widetilde{\hb}_v = ({\gb}_{v},0,\dots, 0)\in\mathbb{R}^{d+b}$. Then we apply a layer
	$$\hb_{v} = \sigma \Big(\widetilde{\hb}_{v} \mW_0' +
		\sum_{i \in [r]} \sum_{w \in N_i(v)}  (\widetilde{\hb}_{v} + \zb_i') \mW_1 + {\bbb}\Big),$$

	where $\mW_0'\in \mathbb{R}^{(d+b)\times e}$ is the matrix $\mW_0\in \mathbb{R}^{d\times e}$ with $b$ additional all-$0$ rows, while $\zb_i' = (0,\dots, 0, \zb_i)\in \mathbb{R}^{d+b}$.
\end{proof}

Together with Proposition \ref{APP:prop:weak-wl-sep} and Theorem \ref{APP:thm:mrgcn-lower}, this result states that \comp\
architectures based on vector summation or concatenation
are provably weaker in terms of their capacity to distinguish vertices in graphs than the ones that use vector scaling.

\subsection{A comparison between \rgcn\ and \comp architectures} \label{APP:sec:comparison}
We proved that \rgcn\ and \comp with point-wise multiplication have the same power discriminating vertices in (multi-relational) graphs. Here we show that these architectures actually define the same functions on multi-relational graphs.

\begin{theorem}
	\label{APP:thm:rgcn-vs-comp}
	{\em
		The following statements hold:
		\begin{itemize}
			\item
			      For any sequence of parameters $\mathbf{W}_\comp^{(t)}$ for \comp{} with point-wise multiplication, there is a sequence of parameters $\mathbf{W}_{\rgcn\ }^{(t)}$ for \rgcn\ such that
			      for each labeled, multi-relational graph $G = (V(G), R_1(G), \dots, R_r(G), \ell)$ and choice of initial vertex features, we have $\hb_{v,\rgcn}^\tup{t} = \hb_{v,\comp}^\tup{t}$, for each $v$ in $V(G)$.
			\item
			      Conversely,
			      for any sequence of parameters $\mathbf{W}_\rgcn^{(t)}$ for \rgcn, there exists a sequence of parameters $\mathbf{W}_{\rgcn\ }^{(2t)}$ for  \comp with point-wise multiplication such that
			      for each labeled, multi-relational graph $G = (V(G), R_1(G), \dots, R_r(G), \ell)$ and choice of initial vertex features, we have  $\hb_{v,\comp}^\tup{2t} = \hb_{v,\rgcn}^\tup{t}$, for each $v$ in $V(G)$.
		\end{itemize}
	}
\end{theorem}

\begin{proof}
	The first item follows since we can simulate one layer of \comp{} with point-wise multiplication using one layer of \rgcn{}. Indeed, take a layer of the form
	$$\hb_{v} = \sigma \Big({\gb}_{v} \mW_0 +
		\sum_{i \in [r]} \sum_{w \in N_i(v)}  ({\gb}_{w} * \zb_i)   \mW_1 \Big),$$
	where ${\gb}_{u},  \zb_i\in \mathbb{R}^d$.
	This is equivalent to
	$$\hb_{v} = \sigma \Big({\gb}_{v} \mW_0 +
		\sum_{i \in [r]} \sum_{w \in N_i(v)}  {\gb}_{w}   \mW_i \Big),$$
	where $\mW_i = \mLambda_i \mW_1$, where $\mLambda_i\in \mathbb{R}^{d\times d}$ is the diagonal matrix whose diagonal is precisely $\zb_i$.

	For the second item, we can simulate one layer of \rgcn{} with two layers of \comp{} with point-wise multiplication. Take a layer
	$$\hb_{v} = \sigma \Big({\gb}_{v} \mW_0 +
		\sum_{i \in [r]} \sum_{w \in N_i(v)}  {\gb}_{w}   \mW_i \Big),$$
	where ${\gb}_{u}\in\mathbb{R}^d$, $\mW_0\in \mathbb{R}^{d\times e}$, $\mW_i \in  \mathbb{R}^{d\times e}$.
	We first apply a layer
	$$\widetilde{\hb}_{v} = {\gb}_{v} \mB$$
	where $\mB\in\mathbb{R}^{d\times dr}$ is the concatenation of $r$ copies of the $d\times d$ identity matrix. In particular, $\widetilde{\hb}_{v}\in \mathbb{R}^{dr}$ is the vector ${\gb}_{v}$ repeated $r$ times. Then we use the layer
	$$\hb_{v} = \sigma \Big(\widetilde{\hb}_{v}  \mW_0' +
		\sum_{i \in [r]} \sum_{w \in N_i(v)}  (\widetilde{\hb}_{w} * \zb_i)  \mW_1' \Big),$$
	where $\mW_0'\in\mathbb{R}^{dr\times e}$ is the matrix $\mW_0\in \mathbb{R}^{d\times e}$ with $d(r-1)$ additional all-$0$ rows,  $\mW_1'\in\mathbb{R}^{dr\times e}$ is the (vertical) concatenation of the matrices $\mW_i$ for $i\in[r]$, and $\zb_i\in \mathbb{R}^{dr}$ is the vector with all entries $0$ except for the $d$ positions $(i-1)d+1,\dots,(i-1)d+d$ which contain the value $1$.
\end{proof}

\begin{remark}
	A similar result holds for the case of \comp{} with point-wise summation/subtraction and \comp{} with vector concatenation. The simulations between these two architectures are implicitly given in the proof of Theorem \ref{APP:thm:weak-comp-wl}.
\end{remark}
\medskip
\begin{remark}
	Note that, as a consequence of Theorem \ref{APP:thm:mrgcn-lower}, Proposition \ref{APP:prop:weak-wl-sep} and the first item of Theorem \ref{APP:thm:weak-comp-wl}, there are functions defined by \rgcn{} or \comp{} with point-wise multiplication that cannot be expressed by \comp{} with point-wise summation/subtraction or vector concatenation. This even holds in the non-uniform sense, that is, if we focus on a single labeled multi-relational graph (the one from Proposition \ref{APP:prop:weak-wl-sep}).
\end{remark}

\section{Missing proofs in Section~\ref{sec:more_exp}}\label{APP:ho}

\begin{proposition}[Proposition~\ref{upper} in the main text]\label{APP:upper}
	{\em For all $r \geq 1$, there exists a pair of non-isomorphic graphs $G = (V(G), R_1(G), \dots, R_r(G), \ell)$ and $H= (V(H), R_1(H), \dots, R_r(H), \ell)$ that cannot be distinguished by \rgcn{} or \comp.}
\end{proposition}
\begin{proof}
	We explicitly construct the graphs $G$ and $H$ for $r \geq 2$. To do so, we take a pair of graphs $A$ and $B$, non-distinguishable by \wlone, and transform them into the multi-relational graphs $G$ and $H$. Let $A$ be a cycle on six vertices and $B$ be the disjoint union of two cycles on three vertices. Clearly, the \wlone{} cannot distinguish the two graphs. Now let $V(G) \coloneqq V(A)$ and $V(H) \coloneqq V(B)$. Further, let $R_i(G)  \coloneqq E(A)$ and $R_i(H) \coloneqq E(B)$ for $i$ in $[r]$. Observe that the multi-relational \wlone{} will reach the stable coloring after one iteration and it will not distinguish the multi-relational graphs $G$ and $H$. Hence, by Theorem~\ref{APP:thm:mrgcn-upper}, the result follows.
\end{proof}
\medskip
\begin{proposition}[Theorem~\ref{thm:ho-upper} in the main text]
	{\em	Let $G = (V(G), R_1(G), \dots, R_r(G), \ell)$ be a labeled, multi-relational graph. Then for all $t \geq 0$, $r > 0$, $k \geq 1$, and all choices of $\UPD^\tup{t}$, $\AGG^\tup{t}$, and all $\vec{v}$ and $\vec{w}$ in $V(G)$,
		\begin{equation*}
			C_{k,r}^{(t)}(\vec{v}) = C_{k,r}^{(t)}(\vec{w}) \ \Longrightarrow \ \hb_{\vec{v},k}^\tup{t} = \hb_{\vec{w},k}^\tup{t}.
		\end{equation*}}
\end{proposition}
\begin{proof}[Proof sketch]
	The proof is analogous to the proof of~\citet[Proposition 3]{Mor+2019}.
\end{proof}
\medskip
\begin{proposition}[Theorem~\ref{thm:ho-lower} in the main text]
	\label{APP:thm:ho-lower}
	{\em	Let $G = (V(G), R_1(G), \dots, R_r(G), \ell)$ be a labeled, multi-relational graph.  Then for all $t \geq 0$ and $k \geq 1$, there exists $\UPD^\tup{t}$, $\AGG^\tup{t}$,  such that for all $\vec{v}$ and $\vec{w}$ in $V(G)$,
		\begin{equation*}
			C_{k,r}^{(t)}(\vec{v}) = C_{k,r}^{(t)}(\vec{w}) \ \Longleftrightarrow \ \hb_{\vec{v},k}^\tup{t} = \hb_{\vec{w},k}^\tup{t}.
		\end{equation*}
	}
\end{proposition}
\begin{proof}
	To prove the results, we need to ensure that there exists instantiations of $\UPD^\tup{t}$ and  $\AGG^\tup{t}$ that are injective.
	To show the existence of injective instantiations of $\AGG^\tup{t}$ for $t > 0$, we write $\AGG^\tup{t}$ as
	\begin{align*}
		\AGG_\text{out}^\tup{t} \Bigl( \AGG_{\text{in},1}^\tup{t} \bigl( & \oms \phi(\hb_{\theta_1(\vec{v},w),k}^\tup{t-1}, \zb_i^{(t)})
		\mid w \in N_i(v_1)  \text{ and } i \in [r]   \cms \bigr), \dots,                                                                \\ \AGG_{\text{in},k}^\tup{t} \bigl(&\oms \phi(\hb_{\theta_k(\vec{v},w),k}^\tup{t-1}, \zb_i^{(t)})
		\mid w \in N_i(v_k)  \text{ and } i \in [r] \cms \bigr) \Bigr),
	\end{align*}
	where 	$\AGG_\text{out}^\tup{t}$ and $\AGG_{\text{in},j}^\tup{t}$ for $j$ in $[k]$ may be a differentiable parameterized functions, e.g., neural networks. Observe that we can represent $\AGG_{\text{in},j}^\tup{t}$ as
	\begin{equation*}
		\sum_{i \in [r]} \sum_{w \in N_i(v_j)}  \phi \big(\hb_{\theta_j(\vec{v},w)}^\tup{t-1}, \zb_i^{(t)} \big) \cdot \mW_{1}^{(t)},
	\end{equation*}
	for $j$ in $[k]$, resembling the aggregation of Equation~\ref{eq:comp}, by Theorem~\ref{APP:thm:mrgcn-lower}, the injectiveness of the above aggregation function follows. A similar argument can be made for $\AGG_\text{out}^\tup{t}$ and $\UPD^\tup{t}$, implying the result.
\end{proof}

Moreover, the following result implies that increasing $k$ leads to a strict boost in terms of expressivity of the \localmkwl{} and \krns{k} architectures in terms of distinguishing non-isomorphic multi-relational graphs.
\medskip
\begin{proposition}\label{APP:thm:hier}
	{\em For $k \geq 2 $ and $r \geq 1$, there exists a pair of non-isomorphic multi-relational graphs  $G_r = (V(G_r), R_1(G_r), \dots, R_r(G_r), \ell)$ and $H_r = (V(H_r), R_1(H_r), \dots, R_r(H_r), \ell)$ that can be distinguished by the  $(k+1)$\text{-}\textsf{MLWL} but not by the  $k$\text{-}\textsf{MLWL}.}
\end{proposition}
\begin{proof}
	See Proof~\ref{proof_hier}.
\end{proof}
\medskip
\begin{corollary}[Corollary~\ref{hier_neural} in the main text]
	\label{APP:hier_neural}
	{\em For $k \geq 2 $ and $r \geq 1$, there exists a pair of non-isomorphic multi-relational graphs  $G_r = (V(G_r), R_1(G_r), \dots, R_r(G_r), \ell)$ and $H= (V(H_r), R_1(H_r), \dots, R_r(H_r), \ell)$ such that:
		\begin{itemize}
			\item For all choices of $\UPD^\tup{t}$, $\AGG^\tup{t}$, for $t > 0$, and $\RO$  the \rn{} architecture will not distinguish the graphs $G_r$ and $H_r$.
			\item There exists $\UPD^\tup{t}$, $\AGG^\tup{t}$, for $t > 0$, and $\RO$ such that the $(k+1)$\text{-}\textsf{RN} will distinguish them.
		\end{itemize}}
	\begin{proof}
		Follows from Theorem~\ref{APP:thm:ho-lower} and Theorem~\ref{APP:thm:hier}.
	\end{proof}
\end{corollary}

\begin{corollary}
	There exists a \krn{2} architecture that is strictly more expressive than the \comp{} and the \rgcn\ architecture in terms of distinguishing non-isomorphic graphs.
\end{corollary}
\begin{proof}
	This follows from Corollary~\ref{APP:hier_neural} and the fact that a \krn{2} is capable to distinguish the graphs constructed in the proof
	of~Proposition~\ref{APP:upper}, which follows from the fact that  the $\delta$-$2$-\textsf{LWL} can distinguish the graphs $A$ and $B$; see, e.g., the proof of Lemma 13 in~\cite{Mor+2022a}.
\end{proof}

\subsection{Proof of Proposition~\ref{APP:thm:hier}}\label{proof_hier}

In the following, we outline the proof of Theorem~\ref{APP:thm:hier}. We modify the construction employed in~\citep{Morris2020b}, Appendix C.1.1., where they provide an infinite family of graphs $(G_k, H_k)_{k \in \mathbb{N}}$ such that the \kwl{} does not distinguish $G_k$ and $H_k$, although the \localkwl{} distinguishes $G_k$ and $H_k$. We recall some relevant definitions from their paper.

\paragraph{Construction of $G_k$ and $H_k$.} Let $K$ denote the complete graph on $k+1$ vertices (without any self-loops). The vertices of $K$ are indexed from $0$ to $k$. Let $E(v)$ denote the set of edges incident to $v$ in $K$: clearly, $|E(v)| = k$ for all $v$ in $V(K)$.
We call the elements of $V(K)$ \emph{base vertices}, and the elements of $E(K)$ \emph{base edges}.
Define the graph $G_k$ as follows:
\begin{enumerate}
	\item For the vertex set $V(G_k)$, we add
	      \begin{enumerate}
		      \item[(a)] $(v,S)$ for each $v$ in $V(K)$ and for each \emph{even} subset $S$ of $E(v)$,
		      \item[(b)] two vertices $e^1,e^0$ for each edge $e$ in $E(K)$.
	      \end{enumerate}
	\item For the edge set $E(G_k)$, we add
	      \begin{enumerate}
		      \item[(a)] an edge $\{e^0,e^1\}$ for each $e$ in $ E(K)$,
		      \item[(b)] an edge between $(v,S)$ and $e^1$ if $v$ in $ e$ and $e$ in $ S$,
		      \item[(c)] an edge between $(v,S)$ and $e^0$ if $v$ in $ e$ and $e$ not in $S$,
	      \end{enumerate}
\end{enumerate}

Define a companion graph $H_k$, in a similar manner to $G_k$, with the following exception: in Step 1(a), for the vertex $0$ in $V(K)$, we choose all \emph{odd} subsets of $E(0)$.

\emph{Distance-two-cliques.} A set $S$ of vertices is said to form a \emph{distance-two-clique} if the distance between any two vertices in $S$ is exactly $2$. The following results were shown in \citep{Morris2020b}.
\begin{lemma}[\citep{Morris2020b}]
	The following holds for the graphs $G_k$ and $H_k$ defined above.
	\begin{itemize}
		\item There exists a distance-two-clique of size $(k+1)$ inside $G_k$.
		\item There does not exist a distance-two-clique of size $(k+1)$ inside $H_k$.
	\end{itemize}
	Hence, $G_k$ and $H_k$ are non-isomorphic.
\end{lemma}

\begin{lemma}[\citep{Morris2020b}]\label{lem:neurips}
	The \localkwl{} distinguishes $G_k$ and $H_k$, while the (oblivious) \kwl{} does not distinguish $G_k$ and $H_k$.
\end{lemma}
Moreover, we need the following result showing that the \localkwl{} forms a hierarchy.
\medskip
\begin{lemma}
	For $k\geq 2$, the \localkwl{} distinguishes $G_k$ and $H_k$, while the $\delta$-$(k-1)$-\textsf{LWL} does not distinguish $G_k$ and $H_k$.
\end{lemma}
\begin{proof}
	The fact that \localkwl{} distinguishes the graphs $G_k$ and $H_k$ follows from Lemma~\ref{lem:neurips}. We know argue that the $\delta$-$(k-1)$-\textsf{LWL} does not distinguish the two graphs.
	First, the (oblivious) \kwl{} has the same expressive power in distinguishing non-isomorphic graphs as the non-oblivious or folklore $(k-1)$\text{-}\textsf{WL}; see~\cite{Gro+2021} for details. Hence, it will not distinguish the graphs $G_k$ and $H_k$. The non-oblivious $(k-1)$\text{-}\textsf{WL}~\citep{Gro+2021} uses the following aggregation function
	\begin{equation*}\label{APP:mio}
		M^{(t)}((v_1, \dots, v_{k-1})) \coloneqq \{\!\! \{    (C_k^{(t)}(\theta_1(\vec{v},w)), \dots,  C_k^{(t)}(\theta_{k-1}(\vec{v},w)))   \mid w \in V(G)   \}\!\!\},
	\end{equation*}
	instead of Equation~\ref{APP:mi}. Notice that from $(C_k^{(t)}(\theta_1(\vec{v},w)), \dots,  C_k^{(t)}(\theta_{k-1}(\vec{v},w)))$ we can recover if there is an edge between the vertex $w$ and a vertex $v_j$ for $j$ in $[k-1]$ in the underlying graph. Hence, the non-oblivious $(k-1)$\text{-}\textsf{WL} is at least as powerful as the $\delta$-$(k-1)$-\textsf{LWL}, implying that the $\delta$-$(k-1)$-\textsf{LWL} is weaker than the $\delta$-$k$-\textsf{LWL}.
\end{proof}
We now construct non-isomorphic multi-relational graphs  $G_r = (V(G_r), R_1(G_r), \dots, R_r(G_r), \ell)$ and $H_r = (V(H_r), R_1(H_r), \dots, R_r(H_r), \ell)$ that can be distinguished by the  $(k+1)$\text{-}\textsf{RLWL} but not by the  $k$\text{-}\textsf{RLWL}.

Let $V(G_r) \coloneqq V(G_k)$ and $V(H_r) \coloneqq V(H_k)$. Further, let $R_i(G_r)  \coloneqq E(G_k)$ and $R_i(H_r) \coloneqq E(H_k)$ for $i$ in $[r]$. By a straightforward inductive argument it follows that $M_{\delta}^{(t)}(\vec{v}) = M_{\delta}^{(t)}(\vec{w})$ implies
$M_{r}^{(t)}(\vec{v}) = M_{r}^{(t)}(\vec{w})$ for all $k$-tuples $\vec{v}$ and $\vec{w}$ in $V(G_k)^k$ or $V(H_k)^k$. This finishes the proof.

\section{R-GCN}
\label{app:rgcn}

Additionally, we probe a modification of the \rgcn\ model with an MLP transformation (denoted as \rgcn+MLP) to facilitate parameter sharing between different relation-specific message propagations:
\begin{align*}
	\hb_{v, \rgcn}^\tup{t} \coloneqq \sigma \Big( 	\hb_{v, \rgcn}^\tup{t-1} \cdot  \mW^{(t)}_0 +  \sum_{i \in [r]}\textsf{MLP} \big(  \sum_{{w \in N_i(v)}}\! \hb_{w,\rgcn}^\tup{t-1} \cdot \mW_i^{(t)} \big) \Big)  \in \RR^{e}.
\end{align*}
This modification has a slightly higher count of learnable parameters.

\section{CompGCN}
\label{APP:comp}
The original \comp{} architecture proposed in~\citet{Vas+2020} considers directed graphs with self-loops, and uses an additional sum to differentiate between in-going, out-going, and self-loop edges, a degree-based normalization, and different weight matrices for these three cases, i.e.,
\begin{align*}
	\hb_{v, \comp}^\tup{t} \coloneqq \sigma \Big(\hb_{v, \comp}^\tup{t-1} \mW^{(t)}_0 +
	\sum_{i \in [r]} \sum_{d \in D} \frac{1}{c_{v,w}} \sum_{w \in N^d_i(v)}  \phi \big(\hb_{w, \comp}^\tup{t-1}, \zb_i^{(t)} \big)  \mW_{1,d}^{(t)} \Big)\in \RR^{e},
\end{align*}
where $D \coloneqq \{ \text{in}, \text{out} \}$, representing in-going an out-going edges, respectively. Here, $N^d_i(v)$ is the restriction of $N^d_i(v)$ of $N_i(v)$ to in-going, out-going, and self-loop edges incident to the vertex $v$. Further, $c_{v,w} \coloneqq \sqrt{|N^d_i(v)| \cdot |N^d_i(w)|}$.
The update of the previous vertex state is performed via the self-loop direction which we separate into the term $\hb_{v, \comp}^\tup{t-1} \mW^{(t)}_0$ for the sake of a unified notation.

In the ablation studies, we probe the following modifications and combinations of those.
\begin{itemize}
	\item \comp{} without normalization (\emph{-norm}):
	      \begin{align*}
		      \hb_{v, \comp}^\tup{t} \coloneqq \sigma \Big(\hb_{v, \comp}^\tup{t-1} \mW^{(t)}_0 +
		      \sum_{i \in [r]} \sum_{d \in D} \sum_{w \in N^d_i(v)}  \phi \big(\hb_{w, \comp}^\tup{t-1}, \zb_i^{(t)} \big)  \mW_{1,d}^{(t)} \Big)\in \RR^{e},
	      \end{align*}
	\item \comp{} without direction-specific weights  (\emph{-dir}):
	      \begin{align*}
		      \hb_{v, \comp}^\tup{t} \coloneqq \sigma \Big(\hb_{v, \comp}^\tup{t-1} \mW^{(t)}_0 +
		      \sum_{i \in [r]} \frac{1}{c_{v,w}} \sum_{w \in N_i(v)}  \phi \big(\hb_{w, \comp}^\tup{t-1}, \zb_i^{(t)} \big)  \mW_{1}^{(t)} \Big)\in \RR^{e},
	      \end{align*}
	\item \comp{} without relations update: $\mathbf{z}^{t+1}_i = \mathbf{z}^t_i$ (\emph{-rp}).
\end{itemize}

As a composition function $\phi(h_w, \mathbf{z}_i)$ we probe several element-wise functions and an MLP:
\begin{itemize}
	\item \texttt{add}: $\phi(\mathbf{h}_w, \mathbf{z}_i) = \mathbf{h}_w + \mathbf{z}_i$ -- element-wise addition
	\item \texttt{mult}: $\phi(\mathbf{h}_w, \mathbf{z}_i) = \mathbf{h}_w * \mathbf{z}_i$ -- element-wise multiplication (Hadamard product)
	\item \texttt{rotate} \citep{sun2019rotate}: $\phi(\mathbf{h}_w, \mathbf{z}_i) = \mathbf{h}_w \odot \mathbf{z}_i$ -- rotation in complex space
	\item \texttt{MLP}: $\phi(\mathbf{h}_w, \mathbf{z}_i) = \text{MLP}([\mathbf{h}_w, \mathbf{z}_i])$ where $[\cdot]$ is column-wise concatenation
\end{itemize}

\section{Datasets and Hyperparameters}
\label{app:data}
Statistics about the datasets are presented in Table~\ref{tab:datasets}.
As neither of the datasets contain an explicit validation set, we retain a random 15\% sample of train vertices for validation and use it to optimize hyperparameters.

\begin{table}[!htp]
	\centering
	\caption{Vertex classification datasets statistics.}\label{tab:datasets}
	\begin{tabular}{lrrrrrrr}\toprule
		Dataset & Vertices  & Edges     & Relations & Train vertices & Test vertices & Classes \\ \midrule
		AIFB    & 8,285     & 29,043    & 45        & 140            & 36            & 4       \\
		AM      & 1,666,764 & 5,988,321 & 133       & 802            & 198           & 11      \\
		\bottomrule
	\end{tabular}
\end{table}

Final hyperparameters are listed in Table~\ref{tab:hyperparam1}, the total parameter count for all trained models is presented in Table~\ref{tab:param_count}.
Due to the size of the AM graph and identified stability of the initial vertex feature dimension, we only train models with dimension $d=4$ on AM.

\begin{table}[!htp]\centering
	\caption{Hyperparameters}\label{tab:hyperparam1}
	\scriptsize
	\begin{tabular}{lccccccc}\toprule
		             & \multicolumn{3}{c}{AIFB}   & \multicolumn{3}{c}{AM}                                            \\\cmidrule{2-4} \cmidrule{5-7}
		             & R-GCN                      & R-GCN + MLP             & CompGCN & R-GCN & R-GCN + MLP & CompGCN \\\midrule
		\# Layers    & 2                          & 2                       & 2       & 3     & 3           & 2       \\
		LR           & 0.001                      & 0.001                   & 0.001   & 0.03  & 0.03        & 0.03    \\
		\# epochs    & 8,000                      & 8,000                   & 8,000   & 100   & 400         & 800     \\
		Dropout      & \multicolumn{3}{c}{0.0}    & \multicolumn{3}{c}{0.0}                                           \\
		Optimizer    & \multicolumn{6}{c}{Adam}                                                                       \\
		Weight decay & \multicolumn{6}{c}{0.0005}                                                                     \\
		\bottomrule
	\end{tabular}
\end{table}

\begin{table}[!htp]\centering
	\caption{Parameter count}\label{tab:param_count}
	\scriptsize
	\begin{tabular}{rrrrrrrr}\toprule
		    & \multicolumn{3}{c}{AIFB} & \multicolumn{3}{c}{AM}                                            \\\cmidrule{2-7}
		dim & R-GCN                    & R-GCN + MLP            & CompGCN & R-GCN  & R-GCN + MLP & CompGCN \\\midrule
		2   & 1,092                    & 1,144                  & 262     &        &             &         \\
		4   & 2,912                    & 2,992                  & 576     & 20,311 & 20,655      & 1,292   \\
		8   & 8,736                    & 8,920                  & 1,168   &        &             &         \\
		16  & 29,120                   & 29,704                 & 3,636   &        &             &         \\
		32  & 104,832                  & 106,984                & 10,852  &        &             &         \\
		64  & 396,032                  & 404,392                & 36,036  &        &             &         \\
		128 & 1,537,576                & 1,570,600              & 129,412 &        &             &         \\
		\bottomrule
	\end{tabular}
\end{table}

\end{document}